\documentclass[11pt]{article}
\usepackage[usenames]{color}
\usepackage{hyperref}
\usepackage{url}
\usepackage{fullpage}
\usepackage{amssymb,amsmath,amsthm}
\usepackage{algorithm}
\usepackage{algpseudocode}
\usepackage{mathptmx}
\usepackage[usenames,dvipsnames]{pstricks}
\usepackage{epsfig}
\usepackage{pst-grad} 
\usepackage{pst-plot} 
\usepackage{authblk}
\usepackage{booktabs}

\newtheorem{invariant}{Invariant}
\newtheorem{fact}{Fact}
\newtheorem{lemma}{Lemma}
\newtheorem{theorem}{Theorem}
\newtheorem{corollary}{Corollary}

\newtheorem{assumption}{Assumption}

\def\dis{\textbf{in\_disagr\_region}}

\def\P{\mathbb{P}}

\def\calF{\mathcal{F}}
\def\calO{\mathcal{O}}
\def\calY{\mathcal{Y}}
\def\calX{\mathcal{X}}
\def\calD{\mathcal{D}}
\def\calH{\mathcal{H}}

\def\calS{\mathcal{S}}
\def\calA{\mathcal{A}}

\def\err{\text{err}}

\def\N{\mathbb{N}}

\def\DIS{\text{DIS}}

\def\B{\text{B}}

\def\disk{\calA_k}
\def\empdisk{\calA_k'}

\def\argmin{\text{argmin}}
\def\learn{\text{CONS-LEARN}}

\title{Active Learning from Weak and Strong Labelers}

\author[1]{Chicheng Zhang\thanks{chz038@eng.ucsd.edu}}
\author[1]{Kamalika Chaudhuri\thanks{kamalika@cs.ucsd.edu}}

\affil[1]{University of California, San Diego}

\begin{document}
\maketitle

\begin{abstract}

An active learner is given a hypothesis class, a large set of unlabeled examples and the ability to interactively query labels to an oracle of a subset of these examples; the goal of the learner is to learn a hypothesis in the class that fits the data well by making as few label queries as possible.

This work addresses active learning with labels obtained from strong and weak labelers, where in addition to the standard active learning setting, we have an extra weak labeler which may occasionally provide incorrect labels. An example is learning to classify medical images where either expensive labels may be obtained from a physician (oracle or strong labeler), or cheaper but occasionally incorrect labels may be obtained from a medical resident (weak labeler). Our goal is to learn a classifier with low error on data labeled by the oracle, while using the weak labeler to reduce the number of label queries made to this labeler. We provide an active learning algorithm for this setting, establish its statistical consistency, and analyze its label complexity to characterize when it can provide label savings over using the strong labeler alone.

\end{abstract}

\section{Introduction}

An active learner is given a hypothesis class, a large set of unlabeled examples and the ability to interactively make label queries to an oracle on a subset of these examples; the goal of the learner is to learn a hypothesis in the class that fits the data well by making as few oracle queries as possible. 

As labeling examples is a tedious task for any one person, many applications of active learning involve synthesizing labels from multiple experts who may have slightly different labeling patterns. While a body of recent empirical work~\cite{Perona, Dy11, Dy12, Grauman1, Grauman2, Carbonell} has developed methods for combining labels from multiple experts, little is known on the theory of actively learning with labels from multiple annotators. For example, what kind of assumptions are needed for methods that use labels from multiple sources to work, when these methods are statistically consistent, and when they can yield benefits over plain active learning are all open questions.  

This work addresses these questions in the context of active learning from strong and weak labelers. Specifically, in addition to unlabeled data and the usual labeling oracle in standard active learning, we have an extra weak labeler. The labeling oracle is a {\em{gold standard}} -- an expert on the problem domain -- and it provides high quality but expensive labels. The weak labeler is cheap, but may provide incorrect labels on some inputs. An example is learning to classify medical images where either expensive labels may be obtained from a physician (oracle), or cheaper but occasionally incorrect labels may be obtained from a medical resident (weak labeler). Our goal is to learn a classifier in a hypothesis class whose error with respect to the data labeled by the oracle is low, while exploiting the weak labeler to reduce the number of queries made to this oracle. Observe that in our model the weak labeler can be incorrect anywhere, and does not necessarily provide uniformly noisy labels everywhere, as was assumed by some previous works~\cite{CKW05, SCS15}.

A plausible approach in this framework is to learn a {\em{difference classifier}} to predict where the weak labeler differs from the oracle, and then use a standard active learning algorithm which queries the weak labeler when this difference classifier predicts agreement. Our first key observation is that this approach is {\em{statistically inconsistent}}; false negative errors (that predict no difference when $O$ and $W$ differ) lead to biased annotation for the target classification task.  We address this problem by learning instead a {\em{cost-sensitive difference classifier}} that ensures that false negative errors rarely occur. Our second key observation is that as existing active learning algorithms usually query labels in localized regions of space, it is sufficient to train the difference classifier restricted to this region and still maintain consistency. This process leads to significant label savings. Combining these two ideas, we get an algorithm that is provably statistically consistent and that works under the assumption that there is a good difference classifier with low false negative error. 

We analyze the label complexity of our algorithm as measured by the number of label requests to the labeling oracle. In general we cannot expect any consistent algorithm to provide label savings under all circumstances, and indeed our worst case asymptotic label complexity is the same as that of active learning using the oracle alone.  Our analysis characterizes when we can achieve label savings, and we show that this happens for example if the weak labeler agrees with the labeling oracle for some fraction of the examples close to the decision boundary. Moreover, when the target classification task is agnostic, the number of labels required to learn the difference classifier is of a lower order than the number of labels required for active learning; thus in realistic cases, learning the difference classifier adds only a small overhead to the total label requirement, and overall we get label savings over using the oracle alone.

\paragraph{Related Work.} There has been a considerable amount of empirical work on active learning where multiple annotators can provide labels for the unlabeled examples. One line of work assumes a generative model for each annotator's labels. The learning algorithm learns the parameters of the individual labelers, and uses them to decide which labeler to query for each example. \cite{Dy11,Dy12,Ding12} consider separate logistic regression models for each annotator, while \cite{weld14, weld15} assume that each annotator's labels are corrupted with a different amount of random classification noise. A second line of work~\cite{Carbonell, ipierotis13} that includes Pro-Active Learning, assumes that each labeler is an expert over an unknown subset of categories, and uses data to measure the class-wise expertise in order to optimally place label queries.  In general, it is not known under what conditions these algorithms are statistically consistent, particularly when the modeling assumptions do not strictly hold, and under what conditions they provide label savings over regular active learning.

\cite{UBS12}, the first theoretical work to consider this problem, consider a model where the weak labeler is more likely to provide incorrect labels in heterogeneous regions of space where similar examples have different labels. Their formalization is orthogonal to ours -- while theirs is more natural in a non-parametric setting, ours is more natural for fitting classifiers in a hypothesis class. In a NIPS 2014 Workshop paper, \cite{CB14} have also considered learning from strong and weak labelers; unlike ours, their work is in the online selective sampling setting, and applies only to linear classifiers and robust regression. \cite{Dekel12} study learning from multiple teachers in the online selective sampling setting in a model where different labelers have different regions of expertise.

Finally, there is a large body of theoretical work~\cite{BBL09, D05, DHM07, H07, ZC14, BL13, BHLZ10} on learning a binary classifier based on interactive label queries made {\em{to a single labeler}}. In the realizable case,~\cite{Now09, D05} show that a generalization of binary search provides an exponential improvement in label complexity over passive learning. The problem is more challenging, however, in the more realistic agnostic case, where such approaches lead to inconsistency. The two styles of algorithms for agnostic active learning are disagreement-based active learning (DBAL)~\cite{BBL09, DHM07, H07, BHLZ10} and the more recent margin-based or confidence-based active learning~\cite{BL13, ZC14}. Our algorithm builds on recent work in DBAL~\cite{BHLZ10, H10}. 

\section{Preliminaries}

\paragraph{The Model.}  We begin with a general framework for actively learning from weak and strong labelers. In the standard active learning setting, we are given unlabelled data drawn from a distribution $U$ over an input space $\calX$, a label space $\calY = \{ -1, 1\}$, a hypothesis class $\calH$, and a labeling oracle $O$ to which we can make interactive queries.

In our setting, we additionally have access to a weak labeling oracle $W$ which we can query interactively. Querying $W$ is significantly cheaper than querying $O$; however, querying $W$ generates a label $y_W$ drawn from a conditional distribution $\P_W(y_W|x)$ which is not the same as the conditional distribution $\P_O(y_O|x)$ of $O$. 

Let $D$ be the data distribution over labelled examples such that: $\P_D(x, y) = \P_U(x) \P_O(y | x)$. Our goal is to learn a classifier $h$ in the hypothesis class $\calH$ such that with probability $\geq 1 - \delta$ over the sample, we have: $\P_D(h(x) \neq y) \leq \min_{h' \in \calH} \P_D(h'(x) \neq y) + \epsilon$, while making as few (interactive) queries to $O$ as possible. 

Observe that in this model $W$ may disagree with the oracle $O$ {\em{anywhere}} in the input space; this is unlike previous frameworks~\cite{CKW05, SCS15} where labels assigned by the weak labeler are corrupted by random classification noise with a higher variance than the labeling oracle. We believe this feature makes our model more realistic.
 
Second, unlike~\cite{UBS12}, mistakes made by the weak labeler {\em{do not have to be close to the decision boundary}}. This keeps the model general and simple, and allows greater flexibility to weak labelers. Our analysis shows that if $W$ is largely incorrect close to the decision boundary, then our algorithm will automatically make more queries to $O$ in its later stages. 

Finally note that $O$ is allowed to be {\em{non-realizable}} with respect to the target hypothesis class $\calH$.

\paragraph{Background on Active Learning Algorithms.} The standard active learning setting is very similar to ours, the only difference being that we have access to the weak oracle $W$. There has been a long line of work on active learning~\cite{BBL09, CAL94, D05, H07, BL13, DHM07, BHLZ10, ZC14}. Our algorithms are based on a style called {\em{disagreement-based active learning (DBAL).}} The main idea is as follows. Based on the examples seen so far, the algorithm maintains a candidate set $V_t$ of classifiers in $\calH$ that is guaranteed with high probability to contain $h^*$, the classifier in $\calH$ with the lowest error. Given a randomly drawn unlabeled example $x_t$, if all classifiers in $V_t$ agree on its label, then this label is inferred; observe that with high probability, this inferred label is $h^*(x_t)$. Otherwise, $x_t$ is said to be in the {\em{disagreement region}} of $V_t$, and the algorithm queries $O$ for its label. $V_t$ is updated based on $x_t$ and its label, and algorithm continues.

Recent works in DBAL~\cite{DHM07, BHLZ10} have observed that it is possible to determine if an $x_t$ is in the disagreement region of $V_t$ without explicitly maintaining $V_t$. Instead, a labelled dataset $S_t$ is maintained; the labels of the examples in $S_t$ are obtained by either querying the oracle or direct inference. To determine whether an $x_t$ lies in the disagreement region of $V_t$, two constrained ERM procedures are performed; empirical risk is minimized over $S_t$ while constraining the classifier to output the label of $x_t$ as $1$ and $-1$ respectively. If these two classifiers have similar training errors, then $x_t$ lies in the disagreement region of $V_t$; otherwise the algorithm infers a label for $x_t$ that agrees with the label assigned by $h^*$.



\paragraph{More Definitions and Notation.} The error of a classifier $h$ under a labelled data distribution $Q$ is defined as: $\err_Q(h) = \P_{(x, y) \sim Q}(h(x) \neq y)$; we use the notation $\err(h, S)$ to denote its empirical error on a labelled data set $S$. We use the notation $h^*$ to denote the classifier with the lowest error under $D$ and $\nu$ to denote its error $\err_D(h^*)$, where $D$ is the target labelled data distribution.  

Our active learning algorithm implicitly maintains a $(1 - \delta)$-{\em{confidence set}} for $h^*$ throughout the algorithm. Given a set $S$ of labelled examples, a set of classifiers $V(S) \subseteq \calH$ is said to be a $(1 - \delta)$-confidence set for $h^*$ with respect to $S$ if $h^* \in V$ with probability $\geq 1 - \delta$ over $S$. 

The disagreement between two classifiers $h_1$ and $h_2$ under an unlabelled data distribution $U$, denoted by $\rho_U(h_1, h_2)$, is $\P_{x \sim U} (h_1(x) \neq h_2(x))$. Observe that the disagreements under $U$ form a pseudometric over $\calH$. We use $\B_U(h, r)$ to denote a ball of radius $r$ centered around $h$ in this metric. The {\em{disagreement region}} of a set $V$ of classifiers, denoted by $\DIS(V)$, is the set of all examples $x \in \calX$ such that there exist two classifiers $h_1$ and $h_2$ in $V$ for which $h_1(x) \neq h_2(x)$.

\section{Algorithm}

 Our main algorithm is a standard single-annotator DBAL algorithm with a major modification: when the DBAL algorithm makes a label query, we use an extra sub-routine to decide whether this query should be made to the oracle or the weak labeler, and make it accordingly. How do we make this decision? We try to predict if weak labeler differs from the oracle on this example; if so, query the oracle, otherwise, query the weak labeler.

\paragraph{Key Idea 1: Cost Sensitive Difference Classifier.} How do we predict if the weak labeler differs from the oracle? A plausible approach is to learn a {\em{difference classifier}} $h^{df}$ in a hypothesis class $\calH^{df}$ to determine if there is a difference. Our first key observation is when the region where $O$ and $W$ differ cannot be perfectly modeled by $\calH^{df}$, the resulting active learning algorithm is {\em{statistically inconsistent}}. Any false negative errors (that is, incorrectly predicting no difference) made by difference classifier leads to biased annotation for the target classification task, which in turn leads to inconsistency. We address this problem by instead learning a {\em{cost-sensitive difference classifier}} and we assume that a classifier with low false negative error exists in $\calH^{df}$. While training, we constrain the false negative error of the difference classifier to be low, and minimize the number of predicted positives (or disagreements between $W$ and $O$) subject to this constraint. This ensures that the annotated data used by the active learning algorithm has diminishing bias, thus ensuring consistency.

\paragraph{Key Idea 2: Localized Difference Classifier Training.} Unfortunately, even with cost-sensitive training, directly learning a difference classifier accurately is expensive. If $d'$ is the VC-dimension of the difference hypothesis class $\calH^{df}$, to learn a target classifier to excess error $\epsilon$, we need a difference classifier with false negative error $O(\epsilon)$, which, from standard generalization theory, requires $\tilde{O}(d'/\epsilon)$ labels~\cite{BB05, S14}! Our second key observation is that we can save on labels by training the difference classifier in a localized manner -- because the DBAL algorithm that builds the target classifier {\em{only makes label queries}} in the disagreement region of the current confidence set for $h^*$. Therefore we train the difference classifier only on this region and still maintain consistency. Additionally this provides label savings because while training the target classifier to excess error $\epsilon$, we need to train a difference classifier with only $\tilde{O}(d' \phi_k/\epsilon)$ labels where $\phi_k$ is the probability mass of this disagreement region. The localized training process leads to an additional technical challenge: as the confidence set for $h^*$ is updated, its disagreement region changes. We address this through an epoch-based DBAL algorithm, where the confidence set is updated and a fresh difference classifier is trained in each epoch.

\paragraph{Main Algorithm.} Our main algorithm (Algorithm~\ref{alg:main}) combines these two key ideas, and like~\cite{BHLZ10}, implicitly maintains the $(1 - \delta)$-confidence set for $h^*$ by through a labeled dataset $\hat{S}_k$. In epoch $k$, the target excess error is $\epsilon_k \approx \frac{1}{2^k}$, and the goal of Algorithm~\ref{alg:main} is to generate a labeled dataset $\hat{S}_k$ that implicitly represents a $(1 - \delta_k)$-confidence set on $h^*$. Additionally, $\hat{S}_k$ has the property that the empirical risk minimizer over it has excess error $\leq \epsilon_k$. 

A naive way to generate such an $\hat{S}_k$ is by drawing $\tilde{O}(d/\epsilon_k^2)$ labeled examples, where $d$ is the VC dimension of $\calH$. Our goal, however, is to generate $\hat{S}_k$ using a much smaller number of label queries, which is accomplished by Algorithm~\ref{alg:adaptive}. This is done in two ways. First, like standard DBAL, we infer the label of any $x$ that lies {\em{outside}} the disagreement region of the current confidence set for $h^*$. Algorithm~\ref{alg:disagree} identifies whether an $x$ lies in this region. Second, for any $x$ in the disagreement region, we determine whether $O$ and $W$ agree on $x$ using a difference classifier; if there is agreement, we query $W$, else we query $O$. The difference classifier used to determine agreement is retrained in the beginning of each epoch by Algorithm~\ref{alg:traindc}, which ensures that the annotation has low bias. 


The algorithms use a constrained ERM procedure $\learn$. Given a hypothesis class $H$, a labeled dataset $S$ and a set of constraining examples $C$, $\learn_H(C,S)$ returns a classifier in $H$ that minimizes the empirical error on $S$ subject to $h(x_i) = y_i$ for each $(x_i, y_i) \in C$. 

\paragraph{Identifying the Disagreement Region.} Algorithm~\ref{alg:disagree} identifies if an unlabeled example $x$ lies in the disagreement region of the current $(1 - \delta)$-confidence set for $h^*$; recall that this confidence set is implicitly maintained through $\hat{S}_k$. The identification is based on two ERM queries. Let $\hat{h}$ be the empirical risk minimizer on the current labeled dataset $\hat{S}_{k-1}$, and $\hat{h}'$ be the empirical risk minimizer on $\hat{S}_{k-1}$ under the constraint that $\hat{h}'(x) = -\hat{h}(x)$. If the training errors of $\hat{h}$ and $\hat{h}'$ are very different, then, all classifiers with training error close to that of $\hat{h}$ assign the same label to $x$, and $x$ lies outside the current disagreement region.


\paragraph{Training the Difference Classifier.} Algorithm~\ref{alg:traindc} trains a difference classifier on a random set of examples which lies in the disagreement region of the current confidence set for $h^*$. The training process is cost-sensitive, and is similar to~\cite{KKM12, KT14, BB05, S14}. A hard bound is imposed on the false-negative error, which translates to a bound on the annotation bias for the target task. The number of positives (i.e., the number of examples where $W$ and $O$ differ) is minimized subject to this constraint; this amounts to (approximately) minimizing the fraction of queries made to $O$. 

The number of labeled examples used in training is large enough to ensure false negative error $O(\epsilon_k/\phi_k)$ over the disagreement region of the current confidence set; here $\phi_k$ is the probability mass of this disagreement region under $U$. This ensures that the overall annotation bias introduced by this procedure in the target task is at most $O(\epsilon_k)$. As $\phi_k$ is small and typically diminishes with $k$, this requires less labels than training the difference classifier globally which would have required $\tilde{O}(d'/\epsilon_k)$ queries to $O$.

\begin{algorithm}[H]
\caption{Active Learning Algorithm from Weak and Strong Labelers}
\begin{algorithmic}[1]
\State Input: Unlabeled distribution $U$, target excess error $\epsilon$, confidence $\delta$, labeling oracle $O$, weak oracle $W$, hypothesis class $\calH$, hypothesis class for difference classifier $\calH^{df}$. 
\State Output: Classifier $\hat{h}$ in $\calH$.
\State Initialize: initial error $\epsilon_0 = 1$, confidence $\delta_0 = \delta/4$. Total number of epochs $k_0 = \lceil \log \frac{1}{\epsilon} \rceil$.
\State Initial number of examples $n_0 = O(\frac{1}{\epsilon_0^2} (d \ln\frac{1}{\epsilon_0^2} + \ln\frac{1}{\delta_0}))$.
\State Draw a fresh sample and query $O$ for its labels $\hat{S}_0 = \{(x_1, y_1), \ldots, (x_{n_0}, y_{n_0}) \}$. Let $\sigma_0 = \sigma(n_0,\delta_0)$. 
\For{$k = 1,2, \ldots, k_0$}
    \State Set target excess error $\epsilon_k = 2^{-k}$, confidence $\delta_k = \delta/4(k+1)^2$.
   
    \State {\em{\# Train Difference Classifier}}
    \State $\hat{h}^{df}_k$ $\leftarrow$ Call Algorithm~\ref{alg:traindc} with inputs unlabeled distribution $U$, oracles $W$ and $O$, target excess error $\epsilon_k$, confidence $\delta_k/2$, previously labeled dataset $\hat{S}_{k-1}$.
    \State {\em{\# Adaptive Active Learning using Difference Classifier}} 
    \State $\sigma_k,\hat{S}_k$ $\leftarrow$ Call Algorithm~\ref{alg:adaptive} with inputs unlabeled distribution $U$, oracles $W$ and $O$, difference classifier $\hat{h}^{df}_k$, target excess error $\epsilon_k$, confidence $\delta_k/2$, previously labeled dataset $\hat{S}_{k-1}$.
\EndFor
\State \Return $\hat{h} \leftarrow \learn_\calH(\emptyset, \hat{S}_{k_0}).$
\end{algorithmic}
\label{alg:main}
\end{algorithm}

\begin{algorithm}
\caption{Training Algorithm for Difference Classifier}
\begin{algorithmic}[1]
\State Input: Unlabeled distribution $U$, oracles $W$ and $O$, target error $\varepsilon$, hypothesis class $\calH^{df}$, confidence $\delta$, previous labeled dataset $\hat{T}$.
\State Output: Difference classifier $\hat{h}^{df}$.
\State Let $\hat{p}$ be an estimate of $\P_{x \sim U}(\dis(\hat{T},\frac{3\epsilon}{2},x) = 1)$, obtained by calling Algorithm~\ref{alg:adaptivebias} with failure probability $\delta/3$.~\footnotemark

\State Let $U' = \emptyset$, $i=1$, and \begin{equation} m = \frac{64 \cdot 1024 \hat{p}}{\varepsilon}(d'\ln\frac{512 \cdot 1024 \hat{p}}{\varepsilon} + \ln\frac{72}{\delta}) \label{eqn:m1} \end{equation}

\Repeat 
    \State Draw an example $x_i$ from $U$. 
    \If{$\dis(\hat{T}, \frac{3\epsilon}{2}, x_i) = 1$}  {\em~\# $x_i$ is inside the disagreement region}
        \State query both $W$ and $O$ for labels to get $y_{i, W}$ and $y_{i, O}$.
    \EndIf
    \State $U' = U' \cup \{(x_i, y_{i,O}, y_{i,W})\}$
    \State $i = i + 1$
\Until{$|U'| = m$}
\State Learn a classifier $\hat{h}^{df} \in \calH^{df}$ based on the following empirical risk minimizer:
\begin{align}
&\hat{h}^{df} = \argmin_{h^{df} \in \calH^{df}} \sum_{i=1}^m 1(h^{df}(x_i) = +1), \text{ s.t.} & \sum_{i=1}^{m} 1( h^{df}(x_i) = -1 \wedge y_{i,O} \neq y_{i, W}) \leq m \varepsilon/256\hat{p}
\label{eqn:costsensitive}
\end{align}
\State \Return $\hat{h}^{df}$.
\end{algorithmic}
\label{alg:traindc}
\end{algorithm}

\footnotetext{Note that if in Algorithm~\ref{alg:adaptivebias}, the upper confidence bound of $\P_{x \sim U}(\dis(\hat{T},\frac{3\epsilon}{2},x) = 1)$ is lower than $\epsilon/64$, then we can halt Algorithm~\ref{alg:traindc} and return an arbitrary $h^{df}$ in $\calH^{df}$. Using this $h^{df}$ will still guarantee the correctness of Algorithm~\ref{alg:main}.}
\paragraph{Adaptive Active Learning using the Difference Classifier.} Finally, Algorithm~\ref{alg:adaptive} is our main active learning procedure, which generates a labeled dataset $\hat{S}_k$ that is implicitly used to maintain a tighter $(1 - \delta)$-confidence set for $h^*$. Specifically, Algorithm~\ref{alg:adaptive} generates a $\hat{S}_k$ such that the set $V_k$ defined as: 
\[ V_k = \{h: \err(h, \hat{S}_k) - \min_{\hat{h}_k \in \calH} \err(\hat{h}_k, \hat{S}_k) \leq 3\epsilon_k/4 \} \]
has the property that:
\[ \{h: \err_D(h) - \err_D(h^*) \leq \epsilon_k/2 \} \subseteq V_k \subseteq \{h: \err_D(h) - \err_D(h^*) \leq \epsilon_k \} \]

This is achieved by labeling, through inference or query, a large enough sample of unlabeled data drawn from $U$. Labels are obtained from three sources - direct inference (if $x$ lies outside the disagreement region as identified by Algorithm~\ref{alg:disagree}), querying $O$ (if the difference classifier predicts a difference), and querying $W$. How large should the sample be to reach the target excess error? If $\err_D(h^*) = \nu$, then achieving an excess error of $\epsilon$ requires $\tilde{O}(d \nu/\epsilon_k^2)$ samples, where $d$ is the VC dimension of the hypothesis class. As $\nu$ is unknown in advance, we use a doubling procedure in lines 4-14 to iteratively determine the sample size.

\begin{algorithm}
\caption{Adaptive Active Learning using Difference Classifier}
\begin{algorithmic}[1]
\State Input: Unlabeled data distribution $U$, oracles $W$ and $O$, difference classifier $h^{df}$, target excess error $\epsilon$, confidence $\delta$, previous labeled dataset $\hat{T}$.
\State Output: Parameter $\sigma$, labeled dataset $\hat{S}$.
\State Let $\hat{h} = \learn_\calH(\emptyset, \hat{T})$.
\For{$t= 1,2, \ldots, $}
	\State Let $\delta^t = \delta/t(t+1)$. Define: $\sigma(2^t, \delta^t) =  \frac{8}{2^t} (2d\ln \frac{2e2^t}{d} + \ln \frac{24}{\delta^t})$. 
	\State Draw $2^t$ examples from $U$ to form $S^{t,U}$.  
	\For{each $x \in S^{t,U}$}:
	\If{$\dis(\hat{T}, \frac{3\epsilon}{2},x) = 0$} {\em~\# $x$ is inside the agreement region}
		\State Add $(x, \hat{h}(x))$ to $\hat{S}^t$.
	\Else{\em  ~\# $x$ is inside the disagreement region}
		\State If $h^{df}(x) = +1$, query $O$ for the label $y$ of $x$, otherwise query $W$. Add $(x, y)$ to $\hat{S}^t$.
	\EndIf
	\EndFor
	\State Train $\hat{h}^t \leftarrow \learn_{\calH} (\emptyset, \hat{S}^t)$.
	\If{$\sigma(2^t, \delta^t) + \sqrt{\sigma(2^t, \delta^t) \err(\hat{h}^t,\hat{S}^t)} \leq \epsilon/512$}			
	\State $t_0 \leftarrow t$, \textbf{break}
	\EndIf
\EndFor
\State \Return $\sigma \leftarrow \sigma(2^{t_0}, \delta^{t_0})$, $\hat{S} \leftarrow \hat{S}^{t_0}$. 
\end{algorithmic}
\label{alg:adaptive}
\end{algorithm}

\begin{algorithm}
\caption{$\dis(\hat{S}, \tau, x)$: Test if $x$ is in the disagreement region of current confidence set}
\begin{algorithmic}[1]
\State Input: labeled dataset $\hat{S}$, rejection threshold $\tau$, unlabeled example $x$.
\State Output: 1 if $x$ in the disagreement region of current confidence set, 0 otherwise.
\State Train $\hat{h} \leftarrow \learn_\calH(\{\emptyset, \hat{S} \})$.
\State Train $\hat{h}_x' \leftarrow \learn_\calH(\{(x, -\hat{h}(x))\}, \hat{S} \})$.
\If{$\err(\hat{h}_x', \hat{S}) - \err(\hat{h}, \hat{S}) > \tau$} \quad {\em{\# $x$ is in the agreement region}}
	\State \Return 0
\Else \quad {\em{\# $x$ is in the disagreement region}}
	\State \Return 1
\EndIf
\end{algorithmic}
\label{alg:disagree}
\end{algorithm}

\section{Performance Guarantees}
\label{sec:theory}


We now examine the performance of our algorithm, which is measured by the number of label queries made to the oracle $O$. Additionally we require our algorithm to be statistically consistent, which means that the true error of the output classifier should converge to the true error of the best classifier in $\calH$ on the data distribution $D$. 

Since our framework is very general, we cannot expect any statistically consistent algorithm to achieve label savings over using $O$ alone under all circumstances. For example, if labels provided by $W$ are the complete opposite of $O$, no algorithm will achieve both consistency and label savings. We next provide an assumption under which Algorithm~\ref{alg:main} works and yields label savings. 

\paragraph{Assumption.} The following assumption states that difference hypothesis class contains a good cost-sensitive predictor of when $O$ and $W$ differ in the disagreement region of $\B_U(h^*, r)$; a predictor is good if it has low false-negative error and predicts a positive label with low frequency. If there is no such predictor, then we cannot expect an algorithm similar to ours to achieve label savings.

\begin{assumption} \label{ass:boundedfn}
Let $\calD$ be the joint distribution: $\P_{\calD}(x, y_O, y_W) = \P_U(x) \P_W(y_W | x) \P_O(y_O | x)$. For any $r, \eta > 0$, there exists an $h^{df}_{\eta, r} \in \calH^{df}$ with the following properties:
\begin{eqnarray}
& \P_{\calD}(h^{df}_{\eta, r}(x) = -1 , x \in \DIS(\B_U(h^*, r)) , y_O \neq y_W) \leq  \eta \label{eqn:fn} \\
& \P_{\calD} (h^{df}_{\eta, r}(x) = 1 , x \in \DIS(\B_U(h^*, r)))  \leq  \alpha(r, \eta) \label{eqn:pos}
\end{eqnarray}
\end{assumption} 

Note that \eqref{eqn:fn}, which states there is a $h^{df} \in \calH^{df}$ with low false-negative error, is minimally restrictive, and is trivially satisfied if $\calH^{df}$ includes the constant classifier that always predicts $1$. Theorem~\label{thm:consistency} shows that~\eqref{eqn:fn} is sufficient to ensure statistical consistency.


\eqref{eqn:pos} in addition states that the number of positives predicted by the classifier $h^{df}_{\eta, r}$ is upper bounded by $\alpha(r, \eta)$. Note $\alpha(r, \eta) \leq \P_U(\DIS(\B_U(h^*, r)))$ always; performance gain is obtained when $\alpha(r, \eta)$ is lower, which happens when the difference classifier predicts agreement on a significant portion of $\DIS(\B_U(h^*, r))$. 

\paragraph{Consistency.} Provided Assumption~\ref{ass:boundedfn} holds, we next show that Algorithm~\ref{alg:main} is statistically consistent. Establishing consistency is non-trivial for our algorithm as the output classifier is trained on labels from both $O$ and $W$.

\begin{theorem} [Consistency]
\label{thm:consistency}
Let $h^*$ be the classifier that minimizes the error with respect to $D$. If Assumption~\ref{ass:boundedfn} holds, then with probability $\geq 1 - \delta$, the classifier $\hat{h}$ output by Algorithm~\ref{alg:main} satisfies: $\err_D(\hat{h}) \leq \err_D(h^*) + \epsilon$. 
\end{theorem}

\paragraph{Label Complexity.} The label complexity of standard DBAL is measured in terms of the disagreement coefficient. The disagreement coefficient $\theta(r)$ at scale $r$ is defined as: $\theta(r) = \sup_{h \in \calH} \sup_{r' \geq r} \frac{\P_U(\DIS(\B_U(h, r'))}{r'}$; intuitively, this measures the rate of shrinkage of the disagreement region with the radius of the ball $\B_U(h, r)$ for any $h$ in $\calH$. It was shown by~\cite{DHM07} that the label complexity of DBAL for target excess generalization error $\epsilon$ is $\tilde{O}(d\theta (2\nu + \epsilon) (1 + \frac{ \nu^2 }{\epsilon^2}))$ where the $\tilde{O}$ notation hides factors logarithmic in $1/\epsilon$ and $1/\delta$. In contrast, the label complexity of our algorithm can be stated in Theorem~\ref{thm:labelcomplexity}. Here we use the $\tilde{O}$ notation for convenience; we have the same dependence on $\log 1/\epsilon$ and $\log 1/\delta$ as the bounds for DBAL.

\begin{theorem} [Label Complexity] \label{thm:labelcomplexity}
Let $d$ be the VC dimension of $\calH$ and let $d'$ be the VC dimension of $\calH^{df}$. If Assumption~\ref{ass:boundedfn} holds, and if the error of the best classifier in $\calH$ on $D$ is $\nu$, then with probability $\geq 1 - \delta$, the following hold:
\begin{enumerate}
\item The number of label queries made by Algorithm~\ref{alg:main} to the oracle $O$ in epoch $k$ at most:
\begin{eqnarray} \label{eqn:epochklc}
 m_k = \tilde{O}\Big{(}\frac{d (2\nu + \epsilon_{k-1}) (\alpha(2 \nu + \epsilon_{k-1}, \epsilon_{k-1}/1024) + \epsilon_{k-1})}{\epsilon_k^2} + \frac{d' \P(\DIS(B_U(h^*, 2 \nu + \epsilon_{k-1})))}{\epsilon_k} \Big{)}
\end{eqnarray}
\item The total number of label queries made by Algorithm~\ref{alg:main} to the oracle $O$ is at most:
\begin{eqnarray}\label{eqn:lc}
\tilde{O} \Big{(} \sup_{r \geq \epsilon} \frac{\alpha(2 \nu + r, r/1024) + r}{2 \nu + r} \cdot d \left( \frac{\nu^2}{\epsilon^2} + 1 \right) + \theta(2 \nu + \epsilon) d' \left(\frac{\nu}{\epsilon} + 1\right) \Big{)}
\end{eqnarray}
\end{enumerate} 
\end{theorem}

\subsection{Discussion} 

The first terms in~\eqref{eqn:epochklc} and~\eqref{eqn:lc} represent the labels needed to learn the target classifier, and second terms represent the overhead in learning the difference classifier.

In the realistic agnostic case (where $\nu > 0$), as $\epsilon \rightarrow 0$, the second terms are  {\em{lower order}} compared to the label complexity of DBAL. Thus {\em{even if $d'$ is somewhat larger than $d$, fitting the difference classifier does not incur an asymptotically high overhead in the more realistic agnostic case.}} In the realizable case, when $d' \approx d$, the second terms are of the same order as the first; therefore we should use a simpler difference hypothesis class $\calH^{df}$ in this case. We believe that the lower order overhead term comes from the fact that there exists a classifier in $\calH^{df}$ whose false negative error is very low.
 
Comparing Theorem~\ref{thm:labelcomplexity} with the corresponding results for DBAL, we observe that instead of $\theta(2 \nu + \epsilon)$, we have the term $\sup_{r \geq \epsilon} \frac{\alpha(2 \nu + r, r/1024)}{2 \nu + r}$. Since $\sup_{r \geq \epsilon} \frac{\alpha(2 \nu + r, r/1024)}{2 \nu + r} \leq \theta(2 \nu + \epsilon)$, the {\em{worst case}} asymptotic label complexity is the same as that of standard DBAL. This label complexity may be considerably better however if $\sup_{r \geq \epsilon} \frac{\alpha(2 \nu + r, r/1024)}{2 \nu + r}$ is less than the disagreement coefficient. As we expect, this will happen when the region of difference between $W$ and $O$ restricted to the disagreement regions is relatively small, and this region is well-modeled by the difference hypothesis class $\calH^{df}$. 

An interesting case is when the weak labeler differs from $O$ close to the decision boundary and agrees with $O$ away from this boundary. In this case, any consistent algorithm should switch to querying $O$ close to the decision boundary. Indeed in earlier epochs, $\alpha$ is low, and our algorithm obtains a good difference classifier and achieves label savings. In later epochs, $\alpha$ is high, the difference classifiers always predict a difference and the label complexity of the later epochs of our algorithm is the same order as DBAL. In practice, if we suspect that we are in this case, we can switch to plain active learning once $\epsilon_k$ is small enough.


\paragraph{Case Study: Linear Classfication under Uniform Distribution.} We provide a simple example where our algorithm provides a better asymptotic label complexity than DBAL. Let $\calH$ be the class of homogeneous linear separators on the $d$-dimensional unit ball and let $\calH^{df} = \{ h \Delta h' : h, h' \in \calH \}$. Furthermore, let $U$ be the uniform distribution over the unit ball.

Suppose that $O$ is a deterministic labeler such that $\err_D(h^*) = \nu > 0$. Moreover, suppose that $W$ is such that there exists a difference classifier $\bar{h}^{df}$ with false negative error $0$ for which $\P_U(\bar{h}^{df}(x) = 1) \leq g$. Additionally, we assume that $g = o(\sqrt{d} \nu)$; observe that this is not a strict assumption on $\calH^{df}$, as $\nu$ could be as much as a constant. Figure~\ref{fig:example} shows an example in $d=2$ that satisfies these assumptions. In this case, as $\epsilon \rightarrow 0$, Theorem~\ref{thm:labelcomplexity} gives the following label complexity bound.

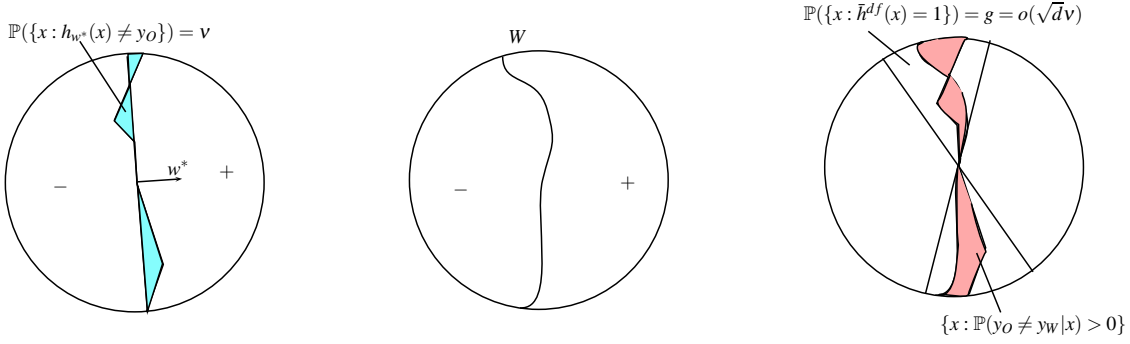
\begin{figure}
\centering
\begin{minipage}{0.32\textwidth}
\scalebox{0.5} 
{
\fontsize{15}{18}\selectfont
\begin{pspicture}(0,-3.8167188)(7.75,3.8367188)
\definecolor{color1312b}{rgb}{0.5215686274509804,0.9921568627450981,0.9921568627450981}
\pscircle[linewidth=0.04,dimen=outer](4.29,-0.33671874){3.46}
\psline[linewidth=0.04](4.49,3.0632813)(3.75,1.3232813)(4.29,0.74328125)(4.35,-0.37671876)(5.05,-2.5167189)(4.61,-3.7567186)
\psline[linewidth=0.04cm](4.09,3.0232813)(4.63,-3.7967188)
\psline[linewidth=0.04cm,arrowsize=0.05291667cm 2.0,arrowlength=1.4,arrowinset=0.4]{->}(4.37,-0.31671876)(5.55,-0.23671874)
\usefont{T1}{ppl}{m}{n}
\rput(6.7345314,-0.04671875){$+$}
\usefont{T1}{ppl}{m}{n}
\rput(2.3145313,-0.44671875){$-$}
\usefont{T1}{ppl}{m}{n}
\rput(5.4445314,0.05328125){$w^*$}
\pspolygon[linewidth=0.04,fillstyle=solid,fillcolor=color1312b](4.11,3.0832813)(4.51,3.1032813)(4.15,2.2632813)
\pspolygon[linewidth=0.04,fillstyle=solid,fillcolor=color1312b](3.75,1.3032813)(4.17,2.2632813)(4.27,0.7632812)
\pspolygon[linewidth=0.04,fillstyle=solid,fillcolor=color1312b](4.37,-0.43671876)(4.63,-3.7567186)(5.03,-2.4967186)
\psline[linewidth=0.04cm](4.05,1.5432812)(2.83,3.4232812)
\usefont{T1}{ppl}{m}{n}
\rput(3.6245313,3.6332812){$\P(\{x: h_{w^*}(x) \neq y_O \}) = \nu$}
\end{pspicture} 
}
\end{minipage}
\hspace{0.31cm}
\begin{minipage}{0.32\textwidth}
\scalebox{0.5} 
{
\fontsize{15}{18}\selectfont
\begin{pspicture}(0,-3.7396533)(6.92,3.7667189)
\pscircle[linewidth=0.04,dimen=outer](3.46,-0.26671875){3.46}
\psbezier[linewidth=0.04](2.5,3.0132813)(2.5,2.2132812)(3.4108603,2.5817487)(3.68,1.6132812)(3.9491398,0.6448137)(3.7955613,0.7051407)(3.56,-0.26671875)(3.3244386,-1.2385782)(3.958598,-3.7196531)(2.96,-3.6667187)
\usefont{T1}{ptm}{m}{n}
\rput(5.824531,-0.35671875){$+$}
\usefont{T1}{ptm}{m}{n}
\rput(1.3845313,-0.5367187){$-$}
\usefont{T1}{ptm}{m}{n}
\rput(2.9045312,3.5632813){$W$}
\end{pspicture} 
}
\end{minipage}
\hspace{-0.8cm}
\begin{minipage}{0.32\textwidth}
\scalebox{0.5} 
{
\fontsize{15}{18}\selectfont
\begin{pspicture}(0,-4.354219)(9.829062,4.354219)
\definecolor{color43b}{rgb}{0.996078431372549,0.6627450980392157,0.6627450980392157}
\pscircle[linewidth=0.04,dimen=outer](5.15,0.06078125){3.46}
\psline[linewidth=0.04](5.43,3.4607813)(4.69,1.7607813)(5.23,1.1807812)(5.29,0.06078125)(5.99,-2.0792189)(5.55,-3.3192186)
\psline[linewidth=0.04cm](3.25,2.9607813)(7.25,-2.7192187)
\psline[linewidth=0.04cm](6.11,3.4407814)(4.39,-3.2992187)
\psline[linewidth=0.04cm](2.87,3.6807814)(3.97,2.7807813)
\usefont{T1}{ppl}{m}{n}
\rput(7.2545314,-4.1292186){$\{x: \P(y_O \neq y_W | x) > 0 \}$}
\usefont{T1}{ppl}{m}{n}
\rput(4.864531,4.150781){$\P(\{x: \bar{h}^{df}(x) = 1\}) = g = o(\sqrt{d} \nu)$}
\psbezier[linewidth=0.04](4.19,3.3407812)(4.19,2.5407813)(5.10086,2.9092488)(5.37,1.9407812)(5.6391397,0.9723137)(5.4855614,1.0326407)(5.25,0.06078125)(5.0144386,-0.9110782)(5.648598,-3.3921533)(4.65,-3.3392189)
\psbezier[linewidth=0.04,fillstyle=solid,fillcolor=color43b](4.17,3.3207812)(4.67,3.4807813)(3.699065,3.4432282)(4.51,2.8007812)(5.3209352,2.1583343)(4.99,2.6007812)(5.01,2.5207813)(5.03,2.4407814)(5.3949165,3.4696276)(5.45,3.4607813)(5.5050836,3.4519348)(5.7221646,3.6057184)(4.73,3.4807813)
\psbezier[linewidth=0.04,fillstyle=solid,fillcolor=color43b](5.05,2.5207813)(4.91,2.2007813)(4.85,2.0007813)(4.73,1.8407812)(4.61,1.6807812)(5.27,1.2007812)(5.27,1.2007812)(5.27,1.2007812)(5.29,0.08078125)(5.31,0.20078126)(5.33,0.32078126)(5.23,-0.07921875)(5.41,0.8407813)(5.59,1.7607813)(5.35,1.9207813)(5.37,2.0607812)
\psbezier[linewidth=0.04,fillstyle=solid,fillcolor=color43b](5.25,0.10078125)(5.21,-0.6592187)(5.22384,-0.39979148)(5.23,-1.3592187)(5.23616,-2.318646)(5.31,-1.5992187)(5.21,-2.5592186)(5.11,-3.5192187)(4.69,-3.3392189)(4.77,-3.3192186)(4.85,-3.2992187)(5.51,-3.4392188)(5.53,-3.3192186)(5.55,-3.1992188)(5.95,-2.2192187)(5.99,-2.1992188)(6.03,-2.1792188)(5.83,-1.9192188)(5.61,-0.89921874)
\psline[linewidth=0.04cm](5.75,-2.2192187)(6.39,-3.7992187)
\end{pspicture} 
}
\end{minipage}
\caption{Linear classification over unit ball with $d=2$. Left: Decision boundary of labeler $O$ and $h^* = h_{w^*}$. The region where $O$ differs from $h^*$ is shaded, and has probability $\nu$. Middle: Decision boundary of weak labeler $W$. Right: $\bar{h}^{df}$, $W$ and $O$. Note that $\{x: \P(y_O \neq y_W|x) > 0 \} \subseteq \{x: \bar{h}^{df}(x) = 1\}$.}

\label{fig:example}
\end{figure}


\begin{corollary} \label{cor:example}
With probability $\geq 1 - \delta$, the number of label queries made to oracle $O$ by Algorithm~\ref{alg:main} is $\tilde{O}\left( d \max(\frac{g}{\nu}, 1) (\frac{\nu^2}{\epsilon^2} + 1) + d^{3/2} \left(1 + \frac{\nu}{\epsilon}\right) \right)$, where the $\tilde{O}$ notation hides factors logarithmic in $1/\epsilon$ and $1/\delta$. 
\end{corollary}
As $g = o(\sqrt{d} \nu)$, this improves over the label complexity of DBAL, which is $\tilde{O}(d^{3/2}(1+\frac{\nu^2}{\epsilon^2}))$. 

\paragraph{Learning with respect to Data labeled by both $O$ and $W$.} Finally, an interesting variant of our model is to measure error relative to data labeled by a mixture of $O$ and $W$ -- say, $(1 - \beta)O + \beta W$ for some $0 < \beta < 1$. Similar measures have been considered in the domain adaptation literature~\cite{BCKPW07}. 

We can also analyze this case using simple modifications to our algorithm and analysis. The results are presented in Corollary~\ref{cor:betamain}, which suggests that the number of label queries to $O$ in this case is roughly $1 - \beta$ times the label complexity in Theorem~\ref{thm:labelcomplexity}.

Let $O'$ be the oracle which, on input $x$, queries $O$ for its label w.p $1 - \beta$ and queries $W$ w.p $\beta$. Let $D'$ be the distribution: $\P_{D'}(x, y) = \P_U(x) \P_{O'} (y | x)$, $h' = \argmin_{h \in \calH} \err_{D'}(h)$ be the classifier in $\calH$ that minimizes error over $D'$, and $\nu' = \err_{D'}(h')$ be its true error. Moreover, suppose that Assumption~\ref{ass:boundedfn} holds with respect to oracles $O'$ and $W$ with $\alpha = \alpha'(r, \eta)$ in~\eqref{eqn:pos}. Then, the modified Algorithm~\ref{alg:main} that simulates $O'$ by random queries to $O$ has the following properties.

\begin{corollary}\label{cor:betamain}
With probability $\geq 1 - 2\delta$,
\begin{enumerate}
\item the classifier $\hat{h}$ output by (the modified) Algorithm~\ref{alg:main} satisfies: $\err_{D'}(\hat{h}) \leq \err_{D'}(h') + \epsilon$. 
\item the total number of label queries made by this algorithm to $O$ is at most:
\begin{eqnarray*}
\tilde{O} \Big{(}  (1-\beta) \Big(\sup_{r \geq \epsilon} \frac{\alpha'(2 \nu' + r, r/1024) + r}{2 \nu' + r} \cdot d \left( \frac{\nu'^2}{\epsilon^2} + 1 \right) + \theta(2 \nu' + \epsilon) d' \left(\frac{\nu'}{\epsilon} + 1\right) \Big) \Big{)}
\end{eqnarray*}
\end{enumerate}
\end{corollary} 

\paragraph{Conclusion.} In this paper, we take a step towards a theoretical understanding of active learning from multiple annotators through a learning theoretic formalization for learning from weak and strong labelers. Our work shows that multiple annotators can be successfully combined to do active learning in a statistically consistent manner under a general setting with few assumptions; moreover, under reasonable conditions, this kind of learning can provide label savings over plain active learning. 

An avenue for future work is to explore a more general setting where we have multiple labelers with expertise on different regions of the input space. Can we combine inputs from such labelers in a statistically consistent manner? Second, our algorithm is intended for a setting where $W$ is biased, and performs suboptimally when the label generated by $W$ is a random corruption of the label provided by $O$. How can we account for both random noise and bias in active learning from weak and strong labelers? 



\subsubsection*{Acknowledgements} We thank NSF under IIS 1162581 for research support and Jennifer Dy for introducing us to the problem of active learning from multiple labelers.

\bibliographystyle{plain}
\bibliography{active}

\newpage
\appendix

\section{Notation}

\subsection{Basic Definitions and Notation}

Here we do a brief recap of notation. We assume that we are given a target hypothesis class $\calH$ of VC dimension $d$, and a difference hypothesis class $\calH^{df}$ of VC dimension $d'$. 

We are given access to an unlabeled distribution $U$ and two labeling oracles $O$ and $W$. Querying $O$ (resp. $W$) with an unlabeled data point $x_i$ generates a label $y_{i, O}$ (resp. $y_{i, W}$) which is drawn from the distribution $\P_O(y | x_i)$ (resp. $\P_W(y | x_i)$). In general these two distributions are different. We use the notation $\calD$ to denote the joint distribution over examples and labels from $O$ and $W$:
\[ \P_{\calD}(x, y_O, y_W) = \P_U(x) \P_O(y_O | x) \P_W(y_W | x) \]

Our goal in this paper is to learn a classifier in $\calH$ which has low error with respect to the data distribution $D$ described as: $\P_D(x, y) = \P_U(x) \P_O(y | x)$ and our goal is use queries to $W$ to reduce the number of queries to $O$. We use $y_O$ to denote the labels returned by $O$, $y_W$ to denote the labels returned by $W$.
 
The error of a classifier $h$ under a labeled data distribution $Q$ is defined as: $\err_Q(h) = \P_{(x, y) \sim Q}(h(x) \neq y)$; we use the notation $\err(h, S)$ to denote its empirical error on a labeled data set $S$. We use the notation $h^*$ to denote the classifier with the lowest error under $D$. Define the excess error of $h$ with respect to distribution $D$ as $\err_D(h) - \err_D(h^*)$. For a set $Z$, we occasionally abuse notation and use $Z$ to also denote the uniform distribution over the elements of $Z$.

\paragraph{Confidence Sets and Disagreement Region.} Our active learning algorithm will maintain a $(1 - \delta)$-{\em{confidence set}} for $h^*$ throughout the algorithm. A set of classifiers $V \subseteq \calH$ produced by a (possibly randomized) algorithm is said to be a $(1 - \delta)$-confidence set for $h^*$ if $h^* \in V$ with probability $\geq 1 - \delta$; here the probability is over the randomness of the algorithm as well as the choice of all labeled and unlabeled examples drawn by it.

Given two classifiers $h_1$ and $h_2$ the disagreement between $h_1$ and $h_2$ under an unlabeled data distribution $U$, denoted by $\rho_U(h_1, h_2)$, is $\P_{x \sim U} (h_1(x) \neq h_2(x))$. Given an unlabeled dataset $S$, the empirical disagreement of $h_1$ and $h_2$ on $S$ is denoted by $\rho_S(h_1, h_2)$. Observe that the disagreements under $U$ form a pseudometric over $\calH$. We use $\B_U(h, r)$ to denote a ball of radius $r$ centered around $h$ in this metric. The {\em{disagreement region}} of a set $V$ of classifiers, denoted by $\DIS(V)$, is the set of all examples $x \in \calX$ such that there exist two classifiers $h_1$ and $h_2$ in $V$ for which $h_1(x) \neq h_2(x)$. 

\paragraph{Disagreement Region.} We denote the disagreement region of a disagreement ball of radius $r$ centered around $h^*$ by 
\begin{equation}
\Delta(r) := \DIS(\B(h^*,r))
\label{eqn:delta}
\end{equation}

\paragraph{Concentration Inequalities.} Suppose $Z$ is a dataset consisting of $n$ iid samples from a distribution $D$. We will use the following result, which is obtained from a standard application of the normalized VC inequality. With probability $1-\delta$ over the random draw of $Z$, for all $h, h' \in \calH$,
\begin{eqnarray}
&&|(\err(h, Z) - \err(h',Z)) - (\err_D(h) - \err_D(h'))| \nonumber \\
& \leq & \min \big(\sqrt{\sigma(n, \delta)\rho_Z(h,h')} + \sigma(n, \delta), \sqrt{\sigma(n, \delta)\rho_D(h,h')} + \sigma(n, \delta)\big)
\label{eqn:normalizedvc}
\end{eqnarray}
\begin{eqnarray}
&&|(\err(h, Z) - \err_D(h)| \nonumber \\
& \leq & \min\big(\sqrt{\sigma(n, \delta)\err(h, Z)} + \sigma(n, \delta), \sqrt{\sigma(n, \delta)\err_D(h)} + \sigma(n, \delta)\big)
\label{eqn:errnormalizedvc}
\end{eqnarray}
where $d$ is the VC dimension of $\calH$ and the notation $\sigma(n,\delta)$ is defined as:
\begin{equation} \label{eqn:defsigma}
\sigma(n,\delta) = \frac{8}{n} (2d\ln \frac{2en}{d} + \ln \frac{24}{\delta}) 
\end{equation}
Equation~\eqref{eqn:normalizedvc} loosely implies the following equation:
\begin{equation}
 |(\err(h, Z) - \err(h', Z)) - (\err_D(h) - \err_D(h'))| \leq \sqrt{4\sigma(n, \delta)}
\label{eqn:addvc} 
\end{equation} 
The following is a consequence of standard Chernoff bounds. Let $X_1, \ldots, X_n$ be iid Bernoulli random variables with mean $p$. If $\hat{p} = \sum_i X_i / n$, then with probabiliy $1-\delta$,
\begin{equation}
|\hat{p} - p| \leq \min(\sqrt{p \gamma(n,\delta)} + \gamma(n,\delta), \sqrt{\hat{p} \gamma(n,\delta)} + \gamma(n,\delta))
\label{eqn:normalizedchernoff}
\end{equation}
where the notation $\gamma(n, \delta)$ is defined as:
\begin{equation} \label{eqn:defgamma}
\gamma(n,\delta) = \frac{4}{n} \ln \frac{2}{\delta}
\end{equation}
Equation~\eqref{eqn:normalizedchernoff} loosely implies the following equation:
\begin{equation}
 |\hat{p} - p| \leq \sqrt{4\gamma(n, \delta)}
\label{eqn:addchernoff} 
\end{equation}


Using the notation we just introduced, we can rephrase Assumption~\ref{ass:boundedfn} as follows. For any $r, \eta > 0$, there exists an $h^{df}_{\eta, r} \in \calH^{df}$ with the following properties:
\begin{eqnarray*}
& \P_{\calD}(h^{df}_{\eta, r}(x) = -1 , x \in \Delta(r) , y_O \neq y_W)  \leq  \eta \\
& \P_{\calD}(h^{df}_{\eta, r}(x) = 1 , x \in \Delta(r))  \leq  \alpha(r, \eta)
\end{eqnarray*}
We end with an useful fact about $\sigma(n, \delta)$. 
\begin{fact}
 The minimum $n$ such that $\sigma(n, \delta/(\log n (\log n + 1) )) \leq \epsilon$ is at most
\[ \frac{64}{\epsilon}(d \ln\frac{512}{\epsilon} + \ln\frac{24}{\delta}) \]
\label{fact:sigma}
\end{fact}


\subsection{Adaptive Procedure for Estimating Probability Mass}

For completeness, we describe in Algorithm~\ref{alg:adaptivebias} a standard doubling procedure for estimating the bias of a coin within a constant factor. This procedure is used by Algorithm~\ref{alg:traindc} to estimate the probability mass of the disagreement region of the current confidence set based on unlabeled examples drawn from $U$.
\begin{algorithm}
\caption{Adaptive Procedure for Estimating the Bias of a Coin}
\begin{algorithmic}[1]
\State Input: failure probability $\delta$, an oracle $\calO$ which returns iid Bernoulli random variables with unknown bias $p$.
\State Output: $\hat{p}$, an estimate of bias $p$ such that $\hat{p} \leq p \leq 2\hat{p}$ with probability $\geq 1 - \delta$.
\For{$i = 1,2,\ldots$}
\State Call the oracle $\calO$ $2^i$ times to get empirical frequency $\hat{p}_i$.
    \If{$\sqrt{\frac{4\ln\frac{4 \cdot 2^i}{\delta}}{2^i}} \leq \hat{p}_i/3$}
        \Return $\hat{p} = \frac{2\hat{p}_i}{3}$
    \EndIf
\EndFor
\end{algorithmic}
\label{alg:adaptivebias}
\end{algorithm}
\begin{lemma}
Suppose $p > 0$ and Algorithm~\ref{alg:adaptivebias} is run with failure probability $\delta$. Then with probability $1-\delta$, (1) the output $\hat{p}$ is such that $\hat{p} \leq p \leq 2\hat{p}$. (2) The total number of calls to $\calO$ is at most $O(\frac{1}{p^2}\ln\frac{1}{\delta p})$.
\label{lem:adaptivebias}
\end{lemma}
\begin{proof}
Consider the event 
\[ E = \{ \text{ for all } i \in \N, |\hat{p}_i - p| \leq \sqrt{\frac{4 \ln\frac{2 \cdot 2^i}{\delta}}{2^i}} \}\]
By Equation~\eqref{eqn:addchernoff} and union bound, $\P(E) \geq 1- \delta$. 
On event $E$, we claim that if $i$ is large enough that 
\begin{equation} 
4\sqrt{\frac{4\ln\frac{4 \cdot 2^i}{\delta}}{2^i}} \leq p 
\label{eqn:suffcond}
\end{equation}
then the condition in line 5 will be met. Indeed, this implies
\[ \sqrt{\frac{4\ln\frac{4 \cdot 2^i}{\delta}}{2^i}} \leq \frac{p - \sqrt{\frac{4\ln\frac{4 \cdot 2^i}{\delta}}{2^i}}}{3} \leq \frac{\hat{p}_i}{3}\]
Define $i_0$ as the smallest number $i$ such that Equation~\eqref{eqn:suffcond} is true. Then by algebra, $2^{i_0} = O(\frac{1}{p^2}\ln\frac{1}{\delta p})$. Hence the number of calls to oracle $\calO$ is at most $1 + 2 + \ldots + 2^{i_0} = O(\frac{1}{p^2}\ln\frac{1}{\delta p})$.\\
Consider the smallest $i^*$ such that the condition in line 5 is met. We have that 
\[ \sqrt{\frac{4\ln\frac{4 \cdot 2^{i^*}}{\delta}}{2^{i^*}}} \leq \hat{p}_{i^*}/3\]
By the definition of $E$,
\[ |p - \hat{p}_{i^*}| \leq \hat{p}_{i^*}/3 \]
that is, $2\hat{p}_{i^*}/3 \leq p \leq 4\hat{p}_{i^*}/3$, implying $\hat{p} \leq p \leq 2\hat{p}$.
\end{proof}

\subsection{Notations on Datasets}
\label{subsubs:dataset}
Without loss of generality, assume the examples drawn throughout Algorithm~\ref{alg:main} have distinct feature values $x$, since this happens with probability 1 under mild assumptions. 

Algorithm~\ref{alg:main} uses a mixture of three kinds of labeled data to learn a target classifier -- labels obtained from querying $O$, labels inferred by the algorithm, and labels obtained from querying $W$. To analyze the effect of these three kinds of labeled data, we need to introduce some notation.

Recall that we define the joint distribution $\calD$ over examples and labels both from $O$ and $W$ as follows:
\[ \P_{\calD}(x, y_O, y_W) = \P_U(x) \P_O(y_O | x) \P_W(y_W | x) \]
where given an example $x$, the labels generated by $O$ and $W$ are conditionally independent.



%
A dataset $\hat{S}$ with empirical error minimizer $\hat{h}$ and a rejection threshold $\tau$ define a implicit confidence set for $h^*$ as follows:
\[ V(\hat{S}, \tau) = \{h: \err(h,\hat{S}) - \err(\hat{h}, \hat{S}) \leq \tau \} \]
At the beginning of epoch $k$, we have $\hat{S}_{k-1}$. $\hat{h}_{k-1}$ is defined as the empirical error minimizer of $\hat{S}_{k-1}$. The disagreement region of the implicit confidence set at epoch $k$, $R_{k-1}$ is defined as $R_{k-1} := \DIS(V(\hat{S}_{k-1}, 3\epsilon_k/2))$. Algorithm~\ref{alg:disagree} $\dis(\hat{S}_{k-1}, 3\epsilon_k/2, x)$ provides a test deciding if an unlabeled example $x$ is inside $R_{k-1}$ in epoch $k$. (See Lemma~\ref{lem:disagree}.) 

Define $\disk$ to be the distribution $\calD$ conditioned on the set $\{(x,y_O,y_W): x \in R_{k-1}\}$. At epoch $k$, Algorithm~\ref{alg:traindc} has inputs distribution $U$, oracles $W$ and $O$, target false negative error $\epsilon = \epsilon_k/128$, hypothesis class $\calH^{df}$, confidence $\delta = \delta_k/2$, previous labeled dataset $\hat{S}_{k-1}$, and outputs a difference classfier $\hat{h}^{df}_k$. By the setting of $m$ in Equation~\eqref{eqn:m1}, Algorithm~\ref{alg:traindc} first computes $\hat{p}_k$ using unlabeled examples drawn from $U$, which is an estimator of $\P_{\calD}(x \in R_{k-1})$.
Then it draws a subsample of size 
\begin{equation} 
m_{k,1} = \frac{64 \cdot 1024 \hat{p}_k}{\epsilon_k} (d\ln\frac{512 \cdot 1024 \hat{p}_k}{\epsilon_k} + \ln\frac{144}{\delta_k})
\label{eqn:mk1}
\end{equation}
iid from $\disk$. We call the resulting dataset $\empdisk$. 

At epoch $k$, Algorithm~\ref{alg:adaptive} performs adaptive subsampling to refine the implicit $(1-\delta)$-confidence set. For each round $t$, it subsamples $U$ to get an unlabeled dataset $S_k^{t,U}$ of size $2^t$. Define the corresponding (hypothetical) dataset with labels queried from both $W$ and $O$ as $\calS_k^t$.
$S_k^t$, the (hypothetical) dataset with labels queried from $O$, is defined as:
\[ S_k^t = \{ (x, y_O ) | (x, y_O, y_W) \in \calS_k^t \} \]
In addition to obtaining labels from $O$, the algorithm obtains labels in two other ways. First, if an $x \in \calX \setminus R_{k-1}$,  then its label is safely inferred and with high probability, this inferred label $\hat{h}_{k-1}(x)$ is equal to $h^*(x)$. Second, if an $x$ lies in $R_{k-1}$ but if the difference classifier $\hat{h}^{df}_k$ predicts agreement between $O$ and $W$, then its label is obtained by querying $W$. The actual dataset $\hat{S}_k^t$ generated by Algorithm~\ref{alg:adaptive} is defined as:
\begin{eqnarray*}
\hat{S}_k^t &=& \{ (x, \hat{h}_{k-1}(x)) | (x, y_O, y_W) \in \calS_k^t, x \notin R_{k-1} \} \cup \{(x,y_O) | (x, y_O, y_W) \in \calS_k^t, x \in R_{k-1}, \hat{h}_k^{df}(x) = +1 \} \\ 
&&\cup \{(x,y_W) | (x, y_O, y_W) \in \calS_k^t, x \in R_{k-1}, \hat{h}_k^{df}(x) = -1 \} 
\end{eqnarray*}
We use $\hat{D}_k$ to denote the labeled data distribution as follows:
\[ \P_{\hat{D}_k}(x,y) = \P_U(x) \P_{\hat{Q}_k}(y|x) \]
\[ \P_{\hat{Q}_k}(y|x) = \begin{cases} I(\hat{h}_{k-1}(x) = y), & x \notin R_{k-1} \\ \P_O(y|x), & x \in R_{k-1}, \hat{h}^{df}_k(x) = +1 \\ \P_W(y|x), & x \in R_{k-1}, \hat{h}^{df}_k(x) = -1 \end{cases}\]
Therefore, $\hat{S}_k^t$ can be seen as a sample of size $2^t$ drawn iid from $\hat{D}_k$.

Observe that $\hat{h}_k^t$ is obtained by training an ERM classifier over $\hat{S}_k^t$, and $\delta_k^t = \delta_k/2t(t+1)$.

Suppose Algorithm~\ref{alg:adaptive} stops at iteration $t_0(k)$, then the final dataset returned is $\hat{S}_k = \hat{S}_k^{t_0(k)}$, with a total number of $m_{k,2}$ label requests to $O$. We define $S_k = S_k^{t_0(k)}$, $\calS_k = \calS_k^{t_0(k)}$ and $\sigma_k = \sigma(2^{t_0(k)}, \delta^{t_0(k)}_k)$.

For $k = 0$, we define the notation $\hat{S}_k$ differently. $\hat{S}_0$ is the dataset drawn iid at random from $D$, with labels queried entirely to $O$. For notational convenience, define $S_0 = \hat{S}_0$. $\sigma_0$ is defined as $\sigma_0 = \sigma(n_0, \delta_0)$, where $\sigma(\cdot, \cdot)$ is defined by Equation~\eqref{eqn:defsigma} and $n_0$ is defined as:
\[ n_0 = (64 \cdot 1024^2) (2d \ln(512 \cdot 1024^2) + \ln\frac{96}{\delta})\]

Recall that $\hat{h}_k = \argmin_{h \in \calH} \err(h, \hat{S}_k)$ is the empirical error minimizer with respect to the dataset $\hat{S}_k$.

Note that the empirical distance $\rho_Z(\cdot, \cdot)$ does not depend on the labels in dataset $Z$, therefore, $\rho_{\hat{S}_k}(h,h') = \rho_{S_k}(h,h')$. We will use them interchangably throughout.


\begin{table}[h!]
  \begin{center}
    \caption{Summary of Notations.}
    \label{tab:notations}
    \begin{tabular}{cp{11cm}c}
      \toprule
      Notation & Explanation & Samples Drawn from\\
      \midrule
      
    $\calD$ & Joint distribution of $(x, y_W, y_O)$ & - \\
    $D$ & Joint distribution of $(x, y_O)$ & - \\
    $U$ & Marginal distribution of $x$ & - \\
    $O$ & Conditional distribution of $y_O$ given $x$ & -\\
    $W$ & Conditional distribution of $y_W$ given $x$ & -\\

    \midrule
      
    $R_{k-1}$ & Disagreement region at epoch $k$ & - \\
    $\calA_k$ & Conditional distribution of $(x, y_W, y_O)$ given $x \in R_{k-1}$ & - \\
    $\calA_k'$ & Dataset used to train difference classifier at epoch $k$ & $\calA_k$ \\
    $h_k^{df}$ & Difference classifier $h_{2\nu+\epsilon_{k-1}, \epsilon_k/512}^{df}$, where $h_{\eta,r}$ is defined in Assumption~\ref{ass:boundedfn} & - \\
    $\hat{h}_k^{df}$ & Difference classifier returned by Algorithm~\ref{alg:traindc} at epoch $k$ & - \\

    \midrule
     
    $S_k^{t,U}$ & unlabeled dataset drawn at iteration $t$ of Algorithm~\ref{alg:adaptive} at epoch $k \geq 1$ & $U$  \\
    $\calS_k^t$ & $S_k^{t, U}$ augmented by labels from $O$ and $W$ & $\calD$\\
    $S_k^t$ & $\{ (x, y_O) | (x, y_O, y_W) \in \calS_k^t \}$ & $D$  \\
    $\hat{S}_k^t$ & Labeled dataset produced at iteration $t$ of Algorithm~\ref{alg:adaptive} at epoch $k \geq 1$ & $\hat{D}_k$ \\ 
    $\hat{D}_k$ & Distribution of $\hat{S}_k^t$ for $k \geq 1$ and any $t$. Has marginal $U$ over $\calX$. The conditional distribution of $y|x$ is $I(h^*(x))$ if $x \notin R_{k-1}$, $W$ if $x \in R_{k-1}$ and $\hat{h}^{df}(x) = -1$, and $O$ otherwise & - \\

    $t_0(k)$ & Number of iterations of Algorithm~\ref{alg:adaptive} at epoch $k \geq 1$ & - \\
    
    \midrule
 
    $\hat{S}_0$ & Initial dataset drawn by Algorithm~\ref{alg:main} & $D$ \\
    $\hat{S}_k$ & Dataset finally returned by Algorithm~\ref{alg:adaptive} at epoch $k \geq 1$. Equal to $\hat{S}_k^{t_0(k)}$ & $\hat{D}_k$ \\
    $S_k$ & Dataset obtained by replacing all labels in $\hat{S}_k$ by labels drawn from $O$. Equal to $S_k^{t_0(k)}$   & $D$ \\
    $\calS_k$ & Equal to $\calS_k^{t_0(k)}$ & $\calD$ \\
    $\hat{h}_k$ & Empirical error minimizer on $\hat{S}_k$  & - \\
    \bottomrule
    
    \end{tabular}
  \end{center}
\end{table}

\subsection{Events}
\label{subs:events}

Recall that $\delta_k = \delta/(4(k+1)^2), \epsilon_k = 2^{-k}$.

Define 
\[ h^{df}_k = h^{df}_{2\nu+\epsilon_{k-1}, \epsilon_k/512} \]
where the notation $h^{df}_{r,\eta}$ is introduced in Assumption~\ref{ass:boundedfn}.

We begin by defining some events that we will condition on later in the proof, and showing that these events occur with high probability. 

Define event
\begin{eqnarray*} 
E_k^1: = \Big\{&& \P_{\calD}(x \in R_{k-1})/2 \leq \hat{p}_k \leq \P_{\calD}(x \in R_{k-1}), \\
&&\text{ and For all $h^{df} \in \calH^{df}$,} \\
&&|\P_{\calA_k'}(h^{df}(x) = -1 , y_O \neq y_W) - \P_{\calA_k}(h^{df}(x) = -1 , y_O \neq y_W)| \leq \frac{\epsilon_k}{1024\P_{\calD}(x \in R_{k-1})} \\
&&+ \sqrt{ \min(\P_{\calA_k}(h^{df}(x) = -1 , y_O \neq y_W), \P_{\calA_k'}(h^{df}(x) = -1 , y_O \neq y_W)) \frac{\epsilon_k}{1024\P_{\calD}(x \in R_{k-1})} }   \\
&&\text{ and }|\P_{\calA_k'}(h^{df}(x) = +1) - \P_{\calA_k}(h^{df}(x) = +1)| \\
&\leq& \sqrt{ \min(\P_{\calA_k}(h^{df}(x) = +1), \P_{\calA_k'}(h^{df}(x) = +1)) \frac{\epsilon_k}{1024\P_{\calD}(x \in R_{k-1})} } + \frac{\epsilon_k}{1024\P_{\calD}(x \in R_{k-1})} \Big\}
\end{eqnarray*}
\begin{fact}
$\P(E_k^1) \geq 1-\delta_k/2$.
\label{fact:ek1}
\end{fact}
Define event 
\begin{eqnarray*} 
E_k^2 = \Big\{  && \text{ For all $t \in \N$, for all $h, h' \in \calH$, } \\
&& |(\err(h, S_k^t) - \err(h', S_k^t)) - (\err_D(h) - \err_D(h'))| \leq \sigma(2^t,\delta_k^t) + \sqrt{\sigma(2^t,\delta_k^t) \rho_{S_k^t}(h,h')} \\
& \text{ and }& \err(h, \hat{S}_k^t) - \err_{\hat{D}_k}(h) \leq \sigma(2^t,\delta_k^t) + \sqrt{\sigma(2^t,\delta_k^t)\err_{\hat{D}_k}(h)} \\
& \text{ and }& \P_{\calS_k^t}(\hat{h}^{df}_k(x) = -1 , y_O \neq y_W , x \in R_{k-1}) - \P_{\calD}(\hat{h}^{df}_k(x)= -1 , y_O \neq y_W , x \in R_{k-1}) \\
&&\leq \sqrt{\gamma(2^t,\delta_k^t)\P_{\calS_k^t}(\hat{h}^{df}_k(x) = -1 , y_O \neq y_W , x \in R_{k-1}) } + \gamma(2^t,\delta_k^t) \\
& \text{ and }& \P_{\calS^t_k}( \hat{h}^{df}_k(x) = -1 \cap x \in R_{k-1}) \leq 2( \P_{\calD}(\hat{h}^{df}_k(x)=-1 , x \in R_{k-1}) + \gamma(2^t,\delta_k^t)) \Big\}
\end{eqnarray*}
\begin{fact}
$\P(E_k^2) \geq 1-\delta_k/2$.
\label{fact:ek2}
\end{fact}
We will also use the following definitions of events in our proof. Define event $F_0$ as  
\[ F_0 =\Big\{ \text{for all } h, h' \in \calH, |(\err(h, S_0) - \err(h', S_0)) - (\err_D(h) - \err_D(h'))| \leq \sigma(n_0,\delta_0) + \sqrt{\sigma(n_0,\delta_0) \rho_{S_0}(h,h')} \Big\} \]
For $k \in \{1, 2, \ldots, k_0\}$, event $F_k$ is defined inductively as
\[ F_k = F_{k-1} \cap (E_k^1 \cap E_k^2) \] 
\begin{fact}
For $k \in \{0,1,\ldots, k_0\}$, $\P(F_k) \geq 1 - \delta_0 - \delta_1 - \ldots - \delta_k$. Specifically, $\P(F_{k_0}) \geq 1-\delta$.
\label{fact:fk}
\end{fact}

The proofs of Facts~\ref{fact:ek1},~\ref{fact:ek2} and~\ref{fact:fk} are provided in Appendix~\ref{sec:remain}.

\section{Proof Outline and Main Lemmas}

The main idea of the proof is to maintain the following three invariants on the outputs of Algorithm~\ref{alg:main} in each epoch. We prove that these invariants hold simultaneously for each epoch with high probability by induction over the epochs. Throughout, for $k \geq 1$, the end of epoch $k$ refers to the end of execution of line 13 of Algorithm~\ref{alg:main} at iteration $k$. The end of epoch 0 refers to the end of execution of line 5 in Algorithm~\ref{alg:main}.

Invariant~\ref{inv:afb} states that if we replace the inferred labels and labels obtained from $W$ in $\hat{S}_k$ by those obtained from $O$ (thus getting the dataset $S_k$), then the excess errors of classifiers in $\calH$ will not decrease by much. 


\begin{invariant}[Approximate Favorable Bias]
Let $h$ be any classifier in $\calH$, and $h'$ be another classifier in $\calH$ with excess error on $D$ no greater than $\epsilon_k$. Then, at the end of epoch $k$, we have:
\[  \err(h, S_k) - \err(h', S_k) \leq \err(h, \hat{S}_k) - \err(h', \hat{S}_k) + \epsilon_k / 16 \]
\label{inv:afb}
\end{invariant}
Invariant~\ref{inv:conc} establishes that in epoch $k$, Algorithm~\ref{alg:adaptive} selects enough examples so as to ensure that concentration of empirical errors of classifiers in $\calH$ on $S_k$ to their true errors. 
\begin{invariant}[Concentration]
At the end of epoch $k$, $\hat{S}_k$, $S_k$ and $\sigma_k$ are such that:\\
1. For any pair of classifiers $h, h' \in \calH$, it holds that: 
\begin{equation}
 |(\err(h, S_k) - \err(h', S_k)) - (\err_D(h) - \err_D(h'))| \leq \sigma_k + \sqrt{ \sigma_k \rho_{S_k}(h,h')}
\label{eqn:normalizedvcsk}
\end{equation}
2. The dataset $\hat{S}_k$ has the following property:
\begin{equation}
 \sigma_k + \sqrt{ \sigma_k \err(\hat{h}_k,\hat{S}_k)} \leq \epsilon_k/512
\label{eqn:datadependentconcsk}
\end{equation}
\label{inv:conc}
\end{invariant}
Finally, Invariant~\ref{inv:dc} ensures that the difference classifier produced in epoch $k$ has low false negative error on the disagreement region of the $(1 - \delta)$ confidence set at epoch $k$.
\begin{invariant}[Difference Classifier]
At epoch $k$, the difference classifier output by Algorithm~\ref{alg:traindc} is such that
\begin{equation} 
\P_{\calD}(\hat{h}^{df}_k(x) = -1 , y_O \neq y_W , x \in R_{k-1}) \leq \epsilon_k/64
\label{eqn:errguar}
\end{equation}
\begin{equation}
\P_{\calD}(\hat{h}^{df}_k(x) = +1 , x \in R_{k-1}) \leq 6(\alpha(2\nu+\epsilon_{k-1}, \epsilon_k/512)+\epsilon_k/1024)
\label{eqn:uncovguar}
\end{equation}
\label{inv:dc}
\end{invariant}
We will show the following property about the three invariants. Its proof is deferred to Subsection~\ref{subs:puttogether}.


\begin{lemma}
There is a numerical constant $c_0>0$ such that the following holds. The collection of events $\{F_k\}_{k=0}^{k_0}$ is such that for $k \in \{0,1,\ldots,k_0\}$:\\
(1) If $k = 0$, then on event $F_k$, at epoch $k$,\\
\indent(1.1) Invariants 1,2 hold.\\
\indent(1.2) The number of label requests to $O$ is at most $m_0 \leq c_0(d + \ln\frac{1}{\delta})$.\\
(2) If $k \geq 1$, then on event $F_k$, at epoch $k$,\\
\indent(2.1) Invariants 1,2,3 hold.\\
\indent(2.2) the number of label requests to $O$ is at most
\[ m_k \leq c_0 \Big( \frac{(\alpha(2\nu + \epsilon_{k-1}, \epsilon_k/1024) + \epsilon_k) (\nu + \epsilon_k)}{\epsilon_k^2}d (\ln^2\frac{1}{\epsilon_k} + \ln^2\frac{1}{\delta_k}) + \frac{ \P_U(x \in \Delta(2\nu+\epsilon_{k-1}))}{\epsilon_k}(d'\ln\frac{1}{\epsilon_k} + \ln\frac{1}{\delta_k}) \Big) \]
\label{lem:inductive}
\end{lemma}

\subsection{Active Label Inference and Identifying the Disagreement Region}

We begin by proving some lemmas about Algorithm~\ref{alg:disagree} which identifies if an example lies in the disagreement region of the current confidence set. This is done by using a constrained ERM oracle $\learn_{H}(\cdot, \cdot)$ using ideas similar to~\cite{DHM07, H10, BHLZ09, BHLZ10}.


\begin{lemma}
When given as input a dataset $\hat{S}$, a threshold $\tau > 0$, an unlabeled example $x$, Algorithm~\ref{alg:disagree} $\dis$ returns 1 if and only if $x$ lies inside $\DIS(V(\hat{S}, \tau))$.
\label{lem:indisregion}
\end{lemma}
\begin{proof}
($\Rightarrow$) If Algorithm~\ref{alg:disagree} returns 1, then we have found a classifier $\hat{h}_x'$ such that (1) $\hat{h}_x(x) = -\hat{h}(x)$, and (2) $\err(\hat{h}_x', \hat{S}) - \err(\hat{h}, \hat{S}) \leq \tau$, i.e. $\hat{h}_x' \in V(\hat{S}, \tau)$. Therefore, $x$ is in $\DIS(V(\hat{S}, \tau))$.\\
($\Leftarrow$) If $x$ is in $\DIS(V(\hat{S}, \tau))$, then there exists a classifier $h \in \calH$ such that (1) $h(x) = -\hat{h}(x)$ and (2) $\err(h, \hat{S}) - \err(\hat{h}, \hat{S}) \leq \tau$. Hence by definition of $\hat{h}_x'$, $\err(\hat{h}_x', \hat{S}) - \err(\hat{h}, \hat{S}) \leq \tau$. Thus, Algorithm~\ref{alg:disagree} returns 1.
\end{proof}

We now provide some lemmas about the behavior of Algorithm~\ref{alg:disagree} called at epoch $k$.

\begin{lemma}
Suppose Invariants~\ref{inv:afb} and~\ref{inv:conc} hold at the end of epoch $k-1$. If $h \in \calH$ is such that $\err_D(h) \leq \err_D(h^*) + \epsilon_{k-1}/2$, then 
\[ \err(h, \hat{S}_{k-1} ) - \err(\hat{h}_{k-1}, \hat{S}_{k-1}) \leq 3\epsilon_{k-1} / 4  \]
\label{lem:lowtrueexcess}
\end{lemma}

\begin{proof}
If $h \in \calH$ has excess error at most $\epsilon_{k-1}/2$ with respect to $D$, then,
\begin{eqnarray*}
&& \err(h, \hat{S}_{k-1}) - \err(\hat{h}_{k-1}, \hat{S}_{k-1}) \\
&\leq& \err(h, S_{k-1}) - \err(\hat{h}_{k-1}, S_{k-1}) + \epsilon_{k-1}/16 \\
&\leq& \err_D(h) - \err_D(\hat{h}_{k-1}) + \sigma_{k-1} + \sqrt{\sigma_{k-1} \rho_{S_{k-1}}(h, \hat{h}_{k-1})} + \epsilon_{k-1}/16 \\
&\leq& \epsilon_{k-1}/2 + \sigma_{k-1} + \sqrt{\sigma_{k-1} \rho_{S_{k-1}}(h, \hat{h}_{k-1})} + \epsilon_{k-1}/16 \\
&\leq& 9\epsilon_{k-1}/16 + \sigma_{k-1} + \sqrt{\sigma_{k-1} \err(h, \hat{S}_{k-1})} + \sqrt{\sigma_{k-1} \err(\hat{h}_{k-1}, \hat{S}_{k-1})} \\
&\leq& 9\epsilon_{k-1}/16 + \sigma_{k-1} + \sqrt{\sigma_{k-1} \err(h, \hat{S}_{k-1})} + \sqrt{\sigma_{k-1} (\err(\hat{h}_{k-1}, \hat{S}_{k-1}) + 9\epsilon_{k-1}/16)}
\end{eqnarray*}
Where the first inequality follows from Invariant~\ref{inv:afb}, the second inequality from Equation~\eqref{eqn:normalizedvcsk} of Invariant~\ref{inv:conc}, the third inequality from the assumption that $h$ has excess error at most $\epsilon_{k-1}/2$, and the fourth inequality from the triangle inequality, the fifth inequality is by adding a nonnegative number in the last term. Continuing,
\begin{eqnarray*}
&& \err(h, \hat{S}_{k-1}) - \err(\hat{h}_{k-1}, \hat{S}_{k-1}) \\
&\leq& 9\epsilon_{k-1}/16 + 4\sigma_{k-1} + 2\sqrt{\sigma_{k-1}(\err(\hat{h}_{k-1}, \hat{S}_{k-1}) + 9\epsilon_{k-1}/16)} \\
&\leq& 9\epsilon_{k-1}/16 + 4\sigma_{k-1} + 2\sqrt{\sigma_{k-1}\err(\hat{h}_{k-1}, \hat{S}_{k-1})} + 2\sqrt{\epsilon_{k-1} / 512 \cdot 9\epsilon_{k-1}/16} \\
&\leq& 9\epsilon_{k-1}/16 + \epsilon_{k-1} / 32 + 2\sqrt{\epsilon_{k-1} / 512 \cdot 9\epsilon_{k-1}/16} \\
&\leq& 3\epsilon_{k-1}/4
\end{eqnarray*}
Where the first inequality is by simple algebra (by letting $D = \err(h,\hat{S}_{k-1})$, $E = \err(\hat{h}_{k-1},\hat{S}_{k-1}) + 9\epsilon_{k-1}/16$, $F = \sigma_{k-1}$ in $D \leq E + F + \sqrt{DF} + \sqrt{EF} \Rightarrow D \leq E + 4F + 2\sqrt{EF}$), the second inequality is from $\sqrt{A+B} \leq \sqrt{A} + \sqrt{B}$ and $\sigma_{k-1} \leq \epsilon_{k-1} / 512$ which utilizes Equation~\eqref{eqn:datadependentconcsk} of Invariant~\ref{inv:conc}, the third inequality is again by Equation~\eqref{eqn:datadependentconcsk} of Invariant~\ref{inv:conc}, the fourth inequality is by algebra.
\end{proof}

\begin{lemma}
Suppose Invariants~\ref{inv:afb} and ~\ref{inv:conc} hold at the end of epoch $k-1$. Then,
\[ \err_D(\hat{h}_{k-1}) - \err_D(h^*) \leq \epsilon_{k-1}/8 \]
\label{lem:hathgood}
\end{lemma}

\begin{proof}
By Lemma~\ref{lem:lowtrueexcess}, we know that since $h^*$ has excess error $0$ with respect to $D$, 
\begin{equation} 
\err(h^*, \hat{S}_{k-1}) - \err(\hat{h}_{k-1}, \hat{S}_{k-1}) \leq 3\epsilon_{k-1}/4 
\label{eqn:starhgood}
\end{equation}
Therefore,
\begin{eqnarray*} 
&&\err_D(\hat{h}_{k-1}) - \err_D(h^*)  \\
&\leq& \err(\hat{h}_{k-1},S_{k-1}) - \err(h^*,S_{k-1}) + \sigma_{k-1} + \sqrt{\sigma_{k-1} \rho_{S_{k-1}}(\hat{h}_{k-1}, h^*)} \\
&\leq& \err(\hat{h}_{k-1},\hat{S}_{k-1}) - \err(h^*,\hat{S}_{k-1}) + \sigma_{k-1} + \sqrt{\sigma_{k-1} \rho_{S_{k-1}}(\hat{h}_{k-1}, h^*)} + \epsilon_{k-1}/16 \\
&\leq& \epsilon_{k-1}/16 + \sigma_{k-1} + \sqrt{\sigma_{k-1} (\err(\hat{h}_{k-1}, \hat{S}_{k-1}) + \err(h^*, \hat{S}_{k-1}))}  \\
&\leq& \epsilon_{k-1}/16 + \sigma_{k-1} + \sqrt{\sigma_{k-1} (2\err(\hat{h}_{k-1}, \hat{S}_{k-1}) + 3\epsilon_{k-1}/4)}  \\
&\leq& \epsilon_{k-1}/16 + \sigma_{k-1} + \sqrt{2\sigma_{k-1} \err(\hat{h}_{k-1}, \hat{S}_{k-1})} + \sqrt{\epsilon_{k-1} / 512 \cdot 3\epsilon_{k-1}/4}  \\
&\leq& \epsilon_{k-1}/8
\end{eqnarray*}
where the first inequality is from Equation~\eqref{eqn:normalizedvcsk} of Invariant~\ref{inv:conc}, the second inequality uses Invariant~\ref{inv:afb}, the third inequality follows from the optimality of $\hat{h}_{k-1}$ and triangle inequality, the fourth inequality uses Equation~\eqref{eqn:starhgood}, the fifth inequality uses the fact that $\sqrt{A+B} \leq \sqrt{A} + \sqrt{B}$ and $\sigma_{k-1} \leq \epsilon_{k-1} / 512$, which is  from Equation~\eqref{eqn:datadependentconcsk} of Invariant~\ref{inv:conc}, the last inequality again utilizes the Equation~\eqref{eqn:datadependentconcsk} of Invariant~\ref{inv:conc}.
\end{proof}

\begin{lemma}
Suppose Invariants~\ref{inv:afb},~\ref{inv:conc}, and~\ref{inv:dc} hold in epoch $k-1$ conditioned on event $F_{k-1}$. Then conditioned on event $F_{k-1}$, the implicit confidence set $V_{k-1} = V(\hat{S}_{k-1}, 3\epsilon_k/2)$ is such that:\\
(1) If $h \in \calH$ satisfies $\err_D(h) - \err_D(h^*) \leq \epsilon_k$, then $h$ is in $V_{k-1}$.\\ 
(2) If $h \in \calH$ is in $V_{k-1}$, then $\err_D(h) - \err_D(h^*) \leq \epsilon_{k-1}$. Hence $V_{k-1} \subseteq \B_U(h^*, 2\nu + \epsilon_{k-1})$.\\
(3) Algorithm~\ref{alg:disagree}, $\dis$, when run on inputs dataset $\hat{S}_{k-1}$, threshold $3\epsilon_k/2$, unlabeled example $x$, returns 1 if and only if $x$ is in $R_{k-1}$.
\label{lem:disagree}
\end{lemma}

\begin{proof}
(1) Let $h$ be a classifier with $\err_D(h) - \err_D(h^*) \leq \epsilon_k = \epsilon_{k-1}/2$. Then, by Lemma~\ref{lem:lowtrueexcess}, one has $\err(h, \hat{S}_{k-1}) - \err(\hat{h}_{k-1}, \hat{S}_{k-1}) \leq 3\epsilon_{k-1}/4 = 3\epsilon_k/2$. Hence, $h$ is in $V_{k-1}$. 
\\
(2) Fix any $h$ in $V_{k-1}$, by definition of $V_{k-1}$,
\begin{equation}
\err(h, \hat{S}_{k-1}) - \err(\hat{h}_{k-1}, \hat{S}_{k-1}) \leq 3\epsilon_k/2 = 3\epsilon_{k-1}/4 
\label{eqn:implicitdisagree}
\end{equation}
Recall that from Lemma~\ref{lem:hathgood}, 
\[ \err_D(\hat{h}_{k-1}) - \err_D(h^*) \leq \epsilon_{k-1}/8 \]
Thus for classifier $h$, applying Invariant~\ref{inv:afb} by taking $h' := \hat{h}_{k-1}$, we get
\begin{equation}
\err(h, S_{k-1}) - \err(\hat{h}_{k-1}, S_{k-1}) \leq \err(h, \hat{S}_{k-1}) - \err(\hat{h}_{k-1}, \hat{S}_{k-1}) + \epsilon_{k-1} / 32
\label{eqn:hathafb}
\end{equation}
Therefore,
\begin{eqnarray*}
&& \err_D(h) - \err_D(\hat{h}_{k-1}) \\
&\leq& \err(h, S_{k-1}) - \err(\hat{h}_{k-1}, S_{k-1}) + \sigma_{k-1} + \sqrt{\sigma_{k-1} \rho_{S_{k-1}}(h,\hat{h}_{k-1})}\\
&\leq& \err(h, S_{k-1}) - \err(\hat{h}_{k-1}, S_{k-1}) + \sigma_{k-1} + \sqrt{\sigma_{k-1} (\err(h, \hat{S}_{k-1}) + \err(\hat{h}_{k-1}, \hat{S}_{k-1}))}\\
&\leq& \err(h, \hat{S}_{k-1}) - \err(\hat{h}_{k-1}, \hat{S}_{k-1}) + \sigma_{k-1} + \sqrt{\sigma_{k-1} (\err(h, \hat{S}_{k-1}) + \err(\hat{h}_{k-1}, \hat{S}_{k-1}))}  + \epsilon_{k-1}/16 \\
&\leq& 13\epsilon_{k-1}/16 + \sigma_{k-1} +\sqrt{\sigma_{k-1} (2\err(\hat{h}_{k-1}, \hat{S}_{k-1}) + 3\epsilon_{k-1}/4)} \\
&\leq& 13\epsilon_{k-1}/16 + \sigma_{k-1} +\sqrt{2\sigma_{k-1} \err(\hat{h}_{k-1}, \hat{S}_{k-1})} + \sqrt{\epsilon_{k-1}/512 \cdot 3\epsilon_{k-1}/4} \\
&\leq& 7\epsilon_{k-1}/8
\end{eqnarray*}
where the first inequality is from Equation~\eqref{eqn:normalizedvcsk} of Invariant~\ref{inv:conc}, the second inequality uses the fact that $\rho_{\hat{S}_{k-1}}(h,h') = \rho_{S_{k-1}}(h,h') \leq \err(h, \hat{S}_{k-1}) + \err(h', \hat{S}_{k-1})$ for $h, h' \in \calH$, the third inequality uses Equation~\eqref{eqn:hathafb}; 
the fourth inequality is from Equation~\eqref{eqn:implicitdisagree}; the fifth inequality is from the fact that $\sqrt{A+B} \leq \sqrt{A} + \sqrt{B}$ and $\sigma_{k-1} \leq \epsilon_{k-1}/512$, which is from Equation~\eqref{eqn:datadependentconcsk} of Invariant~\ref{inv:conc}, the last inequality again follows from Equation~\eqref{eqn:datadependentconcsk} of Invariant~\ref{inv:conc} and algebra.\\
In conjunction with the fact that $\err_D(\hat{h}_{k-1}) - \err_D(h^*)  \leq \epsilon_{k-1}/8$, this implies
\[ \err_D(h) - \err_D(h^*) \leq \epsilon_{k-1} \]
By triangle inequality, $\rho(h,h^*) \leq 2\nu + \epsilon_{k-1}$, hence $h \in \B_U(h^*, 2\nu + \epsilon_{k-1})$. In summary $V_{k-1} \subseteq \B_U(h^*, 2\nu + \epsilon_{k-1})$.\\
(3) Follows directly from Lemma~\ref{lem:indisregion} and the fact that $R_{k-1}  = \DIS(V_{k-1})$.
\end{proof}

\subsection{Training the Difference Classifier}

Recall that $\Delta(r) = \DIS(\B_U(h^*,r))$ is the disagreement region of the disagreement ball centered around $h^*$ with radius $r$.  
\begin{lemma}[Difference Classifier Invariant]
There is a numerical constant $c_1 > 0$ such that the following holds. Suppose that Invariants~\ref{inv:afb} and~\ref{inv:conc} hold at the end of epoch $k-1$ conditioned on event $F_{k-1}$ and that Algorithm~\ref{alg:traindc} has inputs unlabeled data distribution $U$, oracle $O$, $\epsilon = \epsilon_k / 128$, hypothesis class $\calH^{df}$, $\delta = \delta_k/2$, previous labeled dataset $\hat{S}_{k-1}$. Then conditioned on event $F_k$,\\
(1) $\hat{h}^{df}_k$, the output of Algorithm~\ref{alg:traindc}, maintains Invariant~\ref{inv:dc}.\\
(2)(Label Complexity: Part 1.) The number of label queries made to $O$ is at most 
\[ m_{k,1} \leq c_1 \Big( \frac{\P_U(x \in \Delta(2\nu+\epsilon_{k-1}))}{\epsilon_k} (d'\ln\frac{1}{\epsilon_k} + \ln\frac{1}{\delta_k}) \Big) \]
\label{lem:traindc}
\end{lemma}
\begin{proof}
(1) Recall that $F_k = F_{k-1} \cap E_k^1 \cap E_k^2$, where $E_k^1$, $E_k^2$ are defined in Subsection~\ref{subs:events}. Suppose event $F_k$ happens.

\paragraph{Proof of Equation~\eqref{eqn:errguar}.} Recall that $\hat{h}^{df}_k$ is the optimal solution of optimization problem~\eqref{eqn:costsensitive}. We have by feasibility and the fact that on event $E_k^3$, $2\hat{p}_k \geq \P_{\calD}(x \in R_{k-1})$,
\[ \P_{\empdisk}(\hat{h}^{df}_k(x) = -1 , y_O \neq y_W) \leq \frac{\epsilon_k}{256\hat{p}_k} \leq \frac{\epsilon_k}{128\P_{\calD}(x \in R_{k-1})} \]
By definition of event $E_k^2$, this implies
\begin{eqnarray*} 
&&\P_{\disk}(\hat{h}^{df}_k(x) = -1 , y_O \neq y_W)  \\
&\leq& \P_{\empdisk}(\hat{h}^{df}_k(x) = -1 , y_O \neq y_W) + \sqrt{\P_{\empdisk}(\hat{h}^{df}_k(x) = -1 , y_O \neq y_W)\frac{\epsilon_k}{1024\P_{\calD}(x \in R_{k-1})}} + \frac{\epsilon_k}{1024\P_{\calD}(x \in R_{k-1})} \\
&\leq& \frac{\epsilon_k}{64\P_{\calD}(x \in R_{k-1})} \\
\end{eqnarray*}
Indicating
\[ \P_{\calD}(\hat{h}^{df}_k(x) = -1 , y_O \neq y_W , x \in R_{k-1}) \leq \frac{\epsilon_k}{64} \]

\paragraph{Proof of Equation~\eqref{eqn:uncovguar}.} By definition of $h_k^{df}$ in Subsection~\ref{subs:events}, $h_k^{df}$ is such that:
\[ \P_{\calD}(h_k^{df}(x) = +1 , x \in \Delta(2\nu + \epsilon_{k-1})) \leq \alpha(2\nu + \epsilon_{k-1}, \epsilon_k/512)\]
\[ \P_{\calD}(h_k^{df}(x) = -1 , y_O \neq y_W , x \in \Delta(2\nu + \epsilon_{k-1})) \leq \epsilon_k/512 \]
By item (2) of Lemma~\ref{lem:disagree}, we have $R_{k-1} \subseteq \DIS(\B_U(h^*, 2\nu + \epsilon_{k-1}))$, thus
\begin{equation} 
\P_{\calD}(h_k^{df}(x) = +1 , x \in R_{k-1}) \leq \alpha(2\nu + \epsilon_{k-1}, \epsilon_k/512)
\label{eqn:hkdfsmall}
\end{equation}
\begin{equation}
\P_{\calD}(h_k^{df}(x) = -1 , y_O \neq y_W , x \in R_{k-1}) \leq \epsilon_k/512
\label{eqn:hkdfconsist}
\end{equation}
Equation~\eqref{eqn:hkdfconsist} implies that
\begin{equation} 
\P_{\disk}(h_k^{df}(x) = -1 , y_O \neq y_W) \leq \frac{\epsilon_k}{512\P_{\calD}(x \in R_{k-1})} 
\label{eqn:hkdfconsistak}
\end{equation}
Recall that $\empdisk$ is the dataset subsampled from $\disk$ in line 3 of Algorithm~\ref{alg:traindc}. By definition of event $E_k^1$, we have that for $h_k^{df}$,
\begin{eqnarray*} 
&&\P_{\empdisk}(h_k^{df}(x) = -1 , y_O \neq y_W) \\
&\leq& \P_{\disk}(h_k^{df}(x) = -1 , y_O \neq y_W) + \sqrt{\P_{\disk}(h_k^{df}(x) = -1 , y_O \neq y_W) \frac{\epsilon_k}{1024\P_{\calD}(x \in R_{k-1})} } + \frac{\epsilon_k}{1024\P_{\calD}(x \in R_{k-1})} \\
&\leq& \frac{\epsilon_k}{256\P_{\calD}(x \in R_{k-1})} \leq \frac{\epsilon_k}{256\hat{p}_k}
\end{eqnarray*}
where the second inequality is from Equation~\eqref{eqn:hkdfconsistak}, and the last inequality is from the fact that $\hat{p}_k \leq \P_\calD(x \in R_{k-1})$. Hence, $h^{df}_k$ is a feasible solution to the optimization problem~\eqref{eqn:costsensitive}.
Thus,
\begin{eqnarray*} 
&&\P_{\disk}(\hat{h}_k^{df}(x) = +1)\\ 
&\leq& \P_{\empdisk}(\hat{h}_k^{df}(x) = +1) + \sqrt{\P_{\empdisk}(\hat{h}_k^{df}(x) = +1) \frac{\epsilon_k}{1024\P_{\calD}(x \in R_{k-1})}} + \frac{\epsilon_k}{1024\P_{\calD}(x \in R_{k-1})} \\
&\leq& 2(\P_{\empdisk}(\hat{h}_k^{df}(x) = +1) + \frac{\epsilon_k}{1024\P_{\calD}(x \in R_{k-1})})\\
&\leq& 2(\P_{\empdisk}(h_k^{df}(x) = +1) + \frac{\epsilon_k}{1024\P_{\calD}(x \in R_{k-1})})\\
&\leq& 2((\P_{\disk}(h_k^{df}(x) = +1) + \sqrt{\P_{\disk}(h_k^{df}(x) = +1) \frac{\epsilon_k}{1024\P_{\calD}(x \in R_{k-1})}} + \frac{\epsilon_k}{1024\P_{\calD}(x \in R_{k-1})}) + \frac{\epsilon_k}{1024\P_{\calD}(x \in R_{k-1})}) \\
&\leq& 6(\P_{\disk}(h_k^{df}(x) = +1) + \frac{\epsilon_k}{1024\P_{\calD}(x \in R_{k-1})})
\end{eqnarray*}
where the first inequality is by definition of event $E_k^1$, the second inequality is by algebra, the third inequality is by optimality of $\hat{h}_k^{df}$ in~\eqref{eqn:costsensitive}, $\P_{\empdisk}(\hat{h}_k^{df}(x) = +1) \leq \P_{\empdisk}(h_k^{df}(x) = +1)$, the fourth inequality is by definition of event $E_k^1$, the fifth inequality is by algebra. 

Therefore, 
\begin{equation}
\P_{\calD}(\hat{h}_k^{df}(x) = +1, x \in R_{k-1}) \leq 6(\P_{\calD}(h_k^{df}(x) = +1, x \in R_{k-1}) + \epsilon_k/1024) \leq 6(\alpha(2\nu+\epsilon_{k-1}, \epsilon_k/512)+\epsilon_k/1024)
\end{equation}
where the second inequality follows from Equation~\eqref{eqn:hkdfsmall}. This establishes the correctness of Invariant~\ref{inv:dc}.\\
(2) The number of label requests to $O$ follows from line 3 of Algorithm~\ref{alg:traindc} (see Equation~\eqref{eqn:mk1}). That is, we can choose $c_1$ large enough (independently of $k$), such that
\[ m_{k,1} \leq c_1 \Big( \frac{\P_{\calD}(x \in R_{k-1})}{\epsilon_k}(d'\ln\frac{1}{\epsilon_k} + \ln\frac{1}{\delta_k}) \Big) \leq c_1 \Big( \frac{\P_U(x \in \Delta(2\nu+\epsilon_{k-1}))}{\epsilon_k} (d'\ln\frac{1}{\epsilon_k} + \ln\frac{1}{\delta_k}) \Big)\]
where in the second step we use the fact that on event $F_k$, by item (2) of Lemma~\ref{lem:disagree}, $R_{k-1} \subseteq \DIS(\B_U(h^*, 2\nu + \epsilon_{k-1}))$, thus $\P_{\calD}(x \in R_{k-1}) \leq \P_{\calD}(x \in \Delta(2\nu+\epsilon_{k-1})) = \P_{U}(x \in \Delta(2\nu+\epsilon_{k-1}))$. 
\end{proof}

\subsection{Adaptive Subsampling}
\begin{lemma}
There is a numerical constant $c_2 > 0$ such that the following holds. Suppose Invariants~\ref{inv:afb},~\ref{inv:conc}, and~\ref{inv:dc} hold in epoch $k-1$ on event $F_{k-1}$; Algorithm~\ref{alg:adaptive} receives inputs unlabeled distribution $U$, classifier $\hat{h}_{k-1}$, difference classifier $\hat{h}^{df} = \hat{h}^{df}_k$, target excess error $\epsilon = \epsilon_k$, confidence $\delta = \delta_k/2$, previous labeled dataset $\hat{S}_{k-1}$. Then on event $F_k$, \\
(1) $\hat{S}_k$, the output of Algorithm~\ref{alg:adaptive}, maintains Invariants~\ref{inv:afb} and~\ref{inv:conc}.\\
(2) (Label Complexity: Part 2.) The number of label queries to $O$ in Algorithm~\ref{alg:adaptive} is at most:
\[ m_{k,2} \leq c_2 \Big( \frac{(\nu + \epsilon_k) (\alpha(2\nu+\epsilon_{k-1}, \epsilon_k/512) + \epsilon_k)}{\epsilon_k^2}\cdot d(\ln^2\frac{1}{\epsilon_k} + \ln^2\frac{1}{\delta_k}) \Big) \]
\label{lem:adaptive}
\end{lemma}

\begin{proof}

(1) Recall that $F_k = F_{k-1} \cap E_k^1 \cap E_k^2$, where $E_k^1$, $E_k^2$ are defined in Subsection~\ref{subs:events}. Suppose event $F_k$ happens.
\paragraph{Proof of Invariant~\ref{inv:afb}.} We consider a pair of classifiers $h,h' \in \calH$, where $h$ is an arbitrary classifier in $\calH$ and $h'$ has excess error at most $\epsilon_k$. 

At iteration $t = t_0(k)$ 
of Algorithm~\ref{alg:adaptive}, the breaking criteron in line 14 is met, i.e.
\begin{equation}
\sigma(2^{t_0(k)}, \delta_k^{t_0(k)}) + \sqrt{\sigma(2^{t_0(k)}, \delta_k^{t_0(k)}) \err(\hat{h}^{t_0(k)},\hat{S}_k^{t_0(k)}) } \leq \epsilon_k / 512 
\label{eqn:datadependentconcskt0}
\end{equation}

First we expand the definition of $\err(h,S_k)$ and $\err(h,\hat{S}_k)$ respectively:
\small
\[ \err(h,S_k) = \P_{\calS_k}(\hat{h}^{df}_k(x) = +1 , h(x) \neq y_O , x \in R_{k-1}) + \P_{\calS_k}(\hat{h}^{df}_k(x) = -1 , h(x) \neq y_O , x \in R_{k-1}) + \P_{\calS_k}(h(x) \neq y_O , x \notin R_{k-1}) \]
\[ \err(h,\hat{S}_k) = \P_{\calS_k}(\hat{h}^{df}_k(x) = +1 , h(x) \neq y_O , x \in R_{k-1}) + \P_{\calS_k}(\hat{h}^{df}_k(x) = -1 , h(x) \neq y_W , x \in R_{k-1}) + \P_{\calS_k}(h(x) \neq h^*(x) , x \notin R_{k-1}) \]
\normalsize

where we use the fact that by Lemma~\ref{lem:disagree}, for all examples $x \notin R_{k-1}$, $\hat{h}_{k-1}(x) = h^*(x)$.

We next show that $\P_{\calS_k}(\hat{h}^{df}_k(x) = -1 , h(x) \neq y_O , x \in R_{k-1})$ is close to $\P_{\calS_k}(\hat{h}^{df}_k(x) = -1 , h(x) \neq y_W , x \in R_{k-1})$. 

From Lemma~\ref{lem:traindc}, we know that conditioned on event $F_k$, 
\[ \P_{\calD}(\hat{h}^{df}_k(x) = -1 , y_O \neq y_W , x \in R_{k-1}) \leq \epsilon_k/64\]
In the meantime, from Equation~\eqref{eqn:datadependentconcskt0}, $\gamma(2^{t_0(k)}, \delta_k^{t_0(k)}) \leq \sigma(2^{t_0(k)}, \delta_k^{t_0(k)}) \leq \epsilon_k / 512$.
Recall that $\calS_k = \calS_k^{t_0(k)}$. Therefore, by definition of $E_k^2$, \begin{eqnarray*} 
&&\P_{\calS_k}(\hat{h}^{df}_k(x) = -1 , y_O \neq y_W , x \in R_{k-1}) \\
&\leq& \P_{\calD}(\hat{h}^{df}_k(x) = -1 , y_O \neq y_W , x \in R_{k-1}) + \sqrt{\P_{\calD}(\hat{h}^{df}_k(x) = -1 , y_O \neq y_W , x \in R_{k-1})\gamma(2^{t_0(k)}, \delta_k^{t_0(k)})} + \gamma(2^{t_0(k)}, \delta_k^{t_0(k)}) \\
&\leq& \P_{\calD}(\hat{h}^{df}_k(x) = -1 , y_O \neq y_W , x \in R_{k-1}) + \sqrt{\P_{\calD}(\hat{h}^{df}_k(x) = -1 , y_O \neq y_W , x \in R_{k-1})\epsilon_k/512} + \epsilon_k/512 \\
&\leq& \epsilon_k/32
\end{eqnarray*}
By triangle inequality, for all classifier $h_0 \in \calH$,
\begin{equation} 
|\P_{\calS_k}(\hat{h}^{df}_k(x) = -1 , h_0(x) \neq y_O , x \in R_{k-1}) - \P_{\calS_k}(\hat{h}^{df}_k(x) = -1 , h_0(x) \neq y_W , x \in R_{k-1})| \leq \epsilon_k/32 
\label{eqn:owclose}
\end{equation}
Specifically for $h$ and $h'$, Equation~\eqref{eqn:owclose} hold:
\[ |\P_{\calS_k}(\hat{h}^{df}_k(x) = -1 , h(x) \neq y_O , x \in R_{k-1}) - \P_{\calS_k}(\hat{h}^{df}_k(x) = -1 , h(x) \neq y_W , x \in R_{k-1})| \leq \epsilon_k/32 \] 
\[ |\P_{\calS_k}(\hat{h}^{df}_k(x) = -1 , h'(x) \neq y_O , x \in R_{k-1}) - \P_{\calS_k}(\hat{h}^{df}_k(x) = -1 , h'(x) \neq y_W , x \in R_{k-1})| \leq \epsilon_k/32 \] 
Combining, we get:
\begin{eqnarray}
\label{eqn:owclosediff}
&&(\P_{\calS_k}(\hat{h}^{df}_k(x) = -1 , h(x) \neq y_W , x \in R_{k-1}) - \P_{\calS_k}(\hat{h}^{df}_k(x) = -1 , h'(x) \neq y_W , x \in R_{k-1}))  \\
&-& (\P_{\calS_k}(\hat{h}^{df}_k(x) = -1 , h(x) \neq y_O , x \in R_{k-1}) - \P_{\calS_k}(\hat{h}^{df}_k(x) = -1 , h'(x) \neq y_O , x \in R_{k-1})) \leq \epsilon_k/16 \nonumber
\end{eqnarray}
We now show the labels inferred in the region $\calX \setminus R_{k-1}$ is ``favorable" to the classifiers whose excess error is at most $\epsilon_k/2$.\\
By triangle inequality,
\begin{eqnarray*}
&&\P_{\calS_k}(h(x) \neq y_O , x \notin R_{k-1}) - \P_{\calS_k}(h^*(x) \neq y_O , x \notin R_{k-1}) \leq \P_{\calS_k}(h(x) \neq h^*(x) , x \notin R_{k-1})
\end{eqnarray*}
By Lemma~\ref{lem:disagree}, since $h'$ has excess error at most $\epsilon_k$, $h'$ agrees with $h^*$ on all $x$ inside $\calX \setminus R_{k-1}$ on event $F_{k-1}$, hence $\P_{\calS_k}(h'(x) \neq h^*(x) , x \notin R_{k-1}) = 0$. This gives
\begin{eqnarray} 
&&\P_{\calS_k}(h(x) \neq y_O , x \notin R_{k-1}) - \P_{\calS_k}(h'(x) \neq y_O , x \notin R_{k-1}) \nonumber \\
&\leq& \P_{\calS_k}(h(x) \neq h^*(x) , x \notin R_{k-1}) - \P_{\calS_k}(h'(x) \neq h^*(x) , x \notin R_{k-1}) 
\label{eqn:fb}
\end{eqnarray}
Combining Equations~\eqref{eqn:owclosediff} and~\eqref{eqn:fb}, we conclude that
\[ \err(h, S_k) - \err(h', S_k) \leq \err(h, \hat{S}_k) - \err(h', \hat{S}_k) + \epsilon_k / 16 \]
This establishes the correctness of Invariant~\ref{inv:afb}.\\

\paragraph{Proof of Invariant~\ref{inv:conc}.} Recall by definition of $E_k^2$ 
the following concentration  results hold for all $t \in \N$:
\begin{equation*}
 |(\err(h, S_k^t) - \err(h', S_k^t)) - (\err_D(h) - \err_D(h'))| \leq \sigma(2^t, \delta_k^t) + \sqrt{\sigma(2^t, \delta_k^t) \rho_{S_k^t}(h,h'))}
\end{equation*}
In particular, for iteration $t_0(k)$ we have
\begin{equation*}
 |(\err(h, S_k^{t_0(k)}) - \err(h', S_k^{t_0(k)})) - (\err_D(h) - \err_D(h'))| \leq \sigma(2^{t_0(k)}, \delta_k^{t_0(k)}) + \sqrt{\sigma(2^{t_0(k)}, \delta_k^{t_0(k)}) \rho_{S_k^{t_0(k)}}(h,h')}
\end{equation*}
Recall that $\hat{S}_k = \hat{S}_k^{t_0(k)}$, $\hat{h}_k = \hat{h}_k^{t_0(k)}$, and $\sigma_k = \sigma(2^{t_0(k)}, \delta_k^{t_0(k)})$, hence the above is equivalent to
\begin{equation}
 |(\err(h, S_k) - \err(h', S_k)) - (\err_D(h) - \err_D(h'))| \leq \sigma_k + \sqrt{\sigma_k \rho_{S_k}(h,h')}
\label{eqn:stconcenc2} 
\end{equation}
Equation~\eqref{eqn:stconcenc2} 
establishes the correctness of Equation~\eqref{eqn:normalizedvcsk} of Invariant~\ref{inv:conc}. Equation~\eqref{eqn:datadependentconcsk} of Invariant~\ref{inv:conc} follows from Equation~\eqref{eqn:datadependentconcskt0}. \\\\
(2) 
We define $\tilde{h}_k = \argmin_{h \in \calH}\err_{\hat{D}_k}(h)$, and define $\tilde{\nu}_k$ to be $\err_{\hat{D}_k}(\tilde{h}_k)$. To prove the bound on the number of label requests, we first claim that if $t$ is sufficiently large that
\begin{equation}
\sigma(2^t, \delta_k^t) + \sqrt{\sigma(2^t, \delta_k^t) \tilde{\nu}_k} \leq \epsilon_k/1536
\label{eqn:stoppingsuffcond}
\end{equation}
then the algorithm will satisfy the breaking criterion at line 14 of Algorithm~\ref{alg:adaptive}, that is, for this value of $t$,
\begin{equation} 
\sigma(2^t, \delta_k^t) + \sqrt{\sigma(2^t, \delta_k^t) \err(\hat{h}^t,\hat{S}_k^t) } \leq \epsilon_k / 512 
\label{eqn:datadependentconcskt}
\end{equation}
Indeed, by definition of $E_k^2$, if event $F_k$ happens, 
\begin{eqnarray} 
&& \err( \tilde{h}_k, \hat{S}^t_k) \nonumber \\
&\leq& \err_{\hat{D}_k}(\tilde{h}_k) + \sigma(2^t, \delta_k^t) + \sqrt{\sigma(2^t, \delta_k^t)\err_{\hat{D}_k}(\tilde{h}_k) } \nonumber \\
&=& \tilde{\nu}_k + \sigma(2^t, \delta_k^t) + \sqrt{\sigma(2^t, \delta_k^t)\tilde{\nu}_k } 
\label{eqn:hstargoodtk}
\end{eqnarray}
Therefore,
\begin{eqnarray*}
&&\sigma(2^t, \delta_k^t) + \sqrt{\sigma(2^t, \delta_k^t)\err(\hat{h}^t_k, \hat{S}^t_k) } \\
&\leq& \sigma(2^t, \delta_k^t) + \sqrt{\sigma(2^t, \delta_k^t)\err(\tilde{h}_k, \hat{S}^t_k) } \\
&\leq& \sigma(2^t, \delta_k^t) + \sqrt{\sigma(2^t, \delta_k^t)(2\tilde{\nu}_k + 2\sigma(2^t, \delta_k^t)) } \\
&\leq& 3\sigma(2^t, \delta_k^t) + 2\sqrt{\sigma(2^t, \delta_k^t)\tilde{\nu}_k } \\
&\leq& \epsilon_k/512
\end{eqnarray*}
where the first inequality is from the optimality of $\hat{h}^t_k$, the second inequality is from Equation~\eqref{eqn:hstargoodtk}, the third inequality is by algebra, the last inequality follows from Equation~\eqref{eqn:stoppingsuffcond}. The claim follows.\\
Next, we solve for the minimum $t$ that satisfies~\eqref{eqn:stoppingsuffcond}, which is an upper bound of $t_0(k)$. Fact~\ref{fact:sigma} implies that there is a numerical constant $c_3 > 0$ such that
\[ 2^{t_0(k)} \leq c_3 \frac{\tilde{\nu}_k+\epsilon_k}{\epsilon_k^2} (d\ln\frac{1}{\epsilon_k} + \ln\frac{1}{\delta_k}) ) \] 
Thus, there is a numerical constant $c_4 > 0$ such that
\[ t_0(k) \leq c_4 (\ln d + \ln\frac{1}{\epsilon_k} + \ln\ln\frac{1}{\delta_k}) \]
Hence, there is a numerical constant $c_5 > 0$ (that does not depend on $k$) such that the following holds. If event $F_k$ happens, then the number of label queries made by Algorithm~\ref{alg:adaptive} to $O$ can be bounded as follows:
\begin{eqnarray*} 
m_{k,2}&=&\sum_{t=1}^{t_0(k)} |S^{t,U}_k \cap \{x: \hat{h}^{df}_k(x) = +1\} \cap R_{k-1} | \\
&=&\sum_{t=1}^{t_0(k)} 2^t \P_{\calS^{t}_k}(\hat{h}^{df}_k(x) = +1 , x \in R_{k-1}) \\
&\leq&\sum_{t=1}^{t_0(k)} 2^t (2\P_{\calD}(\hat{h}^{df}_k(x) = +1 , x \in R_{k-1}) + 2 \cdot 4\frac{\ln\frac{2}{\delta^t_k}}{2^t}) \\
&\leq& 4 \cdot 2^{t_0(k)} \P_{\calD}(\hat{h}^{df}_k(x) = +1 , x \in R_{k-1}) + 8 \cdot t_0(k) \ln\frac{2}{\delta^{t_0(k)}_k} \\
&\leq& c_5 \Big( (\frac{(\tilde{\nu}_k + \epsilon_k)\P_{\calD}(\hat{h}^{df}_k(x) = +1 , x \in R_{k-1})}{\epsilon_k^2} + 1) \cdot d(\ln^2\frac{1}{\epsilon_k} + \ln^2\frac{1}{\delta_k}) \Big) \\
&\leq& c_5 \Big( (\frac{(\tilde{\nu}_k + \epsilon_k) \cdot 6(\alpha(2\nu+\epsilon_{k-1}, \epsilon_k/512) + \epsilon_k/1024)}{\epsilon_k^2} + 1) \cdot d(\ln^2\frac{1}{\epsilon_k} + \ln^2\frac{1}{\delta_k}) \Big)
\end{eqnarray*}
where the second equality is from the fact that $|S^{t,U}_k \cap \{x: \hat{h}^{df}_k(x) = -1\} \cap R_{k-1} | = |S_k^{t,U}| \cdot \P_{\calS_k^t}(\hat{h}^{df}_k(x)=-1 , x \in R_{k-1})$, in conjunction with $|S_k^{t,U}| = 2^t$; the first inequality is by definition of $E_k^2$, the second and third inequality is from algebra that $t_0(k) \ln\frac{1}{\delta_k^{t_0(k)}} \leq c_5 d (\ln^2\frac{1}{\epsilon_k} + \ln^2\frac{1}{\delta_k})$ for some constant $c_5 > 0$, along with the choice of $c_2$, the fourth step is from Lemma~\ref{lem:traindc} which states that Invariant~\ref{inv:dc} holds at epoch $k$.\\
What remains to be argued is an upper bound on $\tilde{\nu}_k$. Note that
\small
\begin{eqnarray*}
&& \tilde{\nu}_k \\
&=& \min_{h \in \calH} [\P_{\calD}(\hat{h}^{df}_k(x) = -1 , h(x) \neq y_W , x \in R_{k-1}) + \P_{\calD}(\hat{h}^{df}_k(x) = +1 , h(x) \neq y_O , x \in R_{k-1}) + \P_{\calD}(h(x) \neq h^*(x) , x \notin R_{k-1})] \\
&\leq& \P_{\calD}(\hat{h}^{df}_k(x) = -1 , h^*(x) \neq y_W , x \in R_{k-1}) + \P_{\calD}(\hat{h}^{df}_k(x) = +1 , h^*(x) \neq y_O , x \in R_{k-1}) \\
&\leq& \P_{\calD}(\hat{h}^{df}_k(x) = -1 , h^*(x) \neq y_O , x \in R_{k-1}) + \P_{\calD}(\hat{h}^{df}_k(x) = +1 , h^*(x) \neq y_O , x \in R_{k-1}) + \epsilon_k/64\\
&\leq& \P_{\calD}(\hat{h}^{df}_k(x) = -1 , h^*(x) \neq y_O , x \in R_{k-1}) + \P_{\calD}(\hat{h}^{df}_k(x) = +1 , h^*(x) \neq y_O , x \in R_{k-1}) + \P_{\calD}(h(x) \neq y_O , x \notin R_{k-1}) + \epsilon_k/64\\
&=& \nu + \epsilon_k/64 \\
\end{eqnarray*}
\normalsize
where the first step is by definition of $\err_{\hat{D}_k}(h)$, the second step is by the suboptimality of $h^*$, the third step is by Equation~\eqref{eqn:owclose}, the fourth step is by adding a positive term $\P_{\calD}(h(x) \neq y_O , x \notin R_{k-1})$, the fifth step is by definition of $\err_D(h)$.
Therefore, we conclude that there is a numerical constant $c_2 > 0$, such that $m_{k,2}$, the number of label requests to $O$ in Algorithm~\ref{alg:adaptive} is at most
\[ c_2 \Big( \frac{(\nu + \epsilon_k)(\alpha(2\nu+\epsilon_{k-1}, \epsilon_k/512) + \epsilon_k)}{\epsilon_k^2} \cdot d(\ln^2\frac{1}{\epsilon_k} + \ln^2\frac{1}{\delta_k}) \Big) \]
\end{proof}

\subsection{Putting It Together -- Consistency and Label Complexity}
\label{subs:puttogether}
\begin{proof}[Proof of Lemma~\ref{lem:inductive}]
With foresight, pick $c_0 > 0$ to be a large enough constant. We prove the result by induction.\\
\paragraph{Base case.} Consider $k = 0$. Recall that $F_0$ is defined as
\[ F_0 =\Big\{ \text{for all } h, h' \in \calH, |(\err(h, S_0) - \err(h', S_0)) - (\err_D(h) - \err_D(h'))| \leq \sigma(n_0,\delta_0) + \sqrt{\sigma(n_0,\delta_0) \rho_{S_0}(h,h')} \Big\} \]
Note that by definition in Subsection~\ref{subsubs:dataset}, $\hat{S}_0 = S_0$. Therefore Invariant~\ref{inv:afb} trivially holds. When $F_0$ happens, Equation~\eqref{eqn:normalizedvcsk} of Invariant~\ref{inv:conc} holds, and $n_0$ is such that $\sqrt{\sigma_0} \leq \epsilon_0 / 1024$, thus,
\[ \sigma_0 + \sqrt{\sigma_0 \err(\hat{h}_0, \hat{S}_0)} \leq \epsilon_0 / 512 \]
which establishes the validity of Equation~\eqref{eqn:datadependentconcsk} of Invariant~\ref{inv:conc}.

Meanwhile, the number of label requests to $O$ is
\[ n_0 = 64 \cdot 1024^2 (d \ln(512 \cdot 1024^2) + \ln\frac{96}{\delta}) ) \leq  c_0(d + \ln\frac{1}{\delta}) \]

\paragraph{Inductive case.} Suppose the claim holds for $k' < k$. The inductive hypothesis states that Invariants 1,2,3 hold in epoch $k-1$ on event $F_{k-1}$. By Lemma~\ref{lem:traindc} and Lemma~\ref{lem:adaptive}, Invariants 1,2,3 holds in epoch $k$ on event $F_k$. Suppose $F_k$ happens. By Lemma~\ref{lem:traindc}, there is a numerical constant $c_1 > 0$ such that the number of label queries in Algorithm~\ref{alg:traindc} in line 12 is at most
\[ m_{k,1} \leq c_1 \Big(\frac{\P_U(x \in \Delta(2\nu+\epsilon_{k-1}))}{\epsilon_k}(d' \ln\frac{1}{\epsilon_k} + \ln\frac{1}{\delta_k})\Big) \]
Meanwhile, by Lemma~\ref{lem:adaptive}, there is a numerical constant $c_2 > 0$ such that the number of label queries in Algorithm~\ref{alg:adaptive} in line 14 is at most
\[ m_{k,2} \leq c_2 \Big( \frac{(\alpha(2\nu+\epsilon_{k-1}, \epsilon_k/512) + \epsilon_k) (\nu + \epsilon_k)}{\epsilon_k^2} \cdot d (\ln^2\frac{1}{\epsilon_k} + \ln^2\frac{1}{\delta_k})\Big) \]
Thus, the number of label requests in total at epoch $k$ is at most
\begin{eqnarray*} 
m_k &=& m_{k,1} + m_{k,2} \\
&\leq& c_0 \Big( (\frac{\alpha(2\nu + \epsilon_{k-1}, \epsilon_k/512) + \epsilon_k) (\nu + \epsilon_k)}{\epsilon_k^2}d (\ln^2\frac{1}{\epsilon_k} + \ln^2\frac{1}{\delta_k}) + \frac{ \P_U(x \in \Delta(2\nu+\epsilon_{k-1}))}{\epsilon_k}(d'\ln\frac{1}{\epsilon_k} + \ln\frac{1}{\delta_k}) \Big) 
\end{eqnarray*}
This completes the induction.
\end{proof}

\begin{theorem}[Consistency]
If $F_{k_0}$ happens, then the classifier $\hat{h}$ returned by Algorithm~\ref{alg:main} is such that
\[ \err_D(\hat{h}) - \err_D(h^*) \leq \epsilon \]
\label{thm:consistencyconcrete}
\end{theorem}
\begin{proof}
By Lemma~\ref{lem:inductive}, Invariants~\ref{inv:afb},~\ref{inv:conc},~\ref{inv:dc} hold at epoch $k_0$. Thus by Lemma~\ref{lem:hathgood},
\[ \err_D(\hat{h}) - \err_D(h^*) = \err_D(\hat{h}_{k_0}) - \err_D(h^*) \leq \epsilon_{k_0}/8 \leq \epsilon \]
\end{proof}

\begin{proof}[Proof of Theorem~\ref{thm:consistency}]
This is an immediate consequence of Theorem~\ref{thm:consistencyconcrete}.
\end{proof}

\begin{theorem}[Label Complexity]
If $F_{k_0}$ happens, then the number of label queries made by Algorithm~\ref{alg:main} to $O$ is at most
\[ \tilde{O}((\sup_{r \geq \epsilon} \frac{\alpha(2\nu+r,r/1024)}{2\nu+r}) d(\frac{\nu^2}{\epsilon^2} + 1) + (\sup_{r \geq \epsilon} \frac{\P_U(x \in \Delta(2\nu+r))}{2\nu+r}) d'(\frac{\nu}{\epsilon} + 1)) \]
\label{thm:labelcomplexityconcrete}
\end{theorem}
\begin{proof}
Conditioned on event $F_{k_0}$, we bound the sum $\sum_{k=0}^{k_0} m_k$.

\small
\begin{eqnarray*}
&& \sum_{k=0}^{k_0} m_k \\
&\leq& c_0(d + \ln\frac{1}{\delta}) + c_0 \Big( \sum_{k=1}^{k_0} \frac{(\alpha(2\nu + \epsilon_{k-1}, \epsilon_k/512) + \epsilon_k) (\nu + \epsilon_k)}{\epsilon_k^2} d (\ln^2\frac{1}{\epsilon_k} + \ln^2\frac{1}{\delta_k}) + \frac{ \P_U(x \in \Delta(2\nu+\epsilon_{k-1}))}{\epsilon_k} (d'\ln\frac{1}{\epsilon_k} + \ln\frac{1}{\delta_k}) \Big) \\
&\leq& c_0(d + \ln\frac{1}{\delta}) + c_0 \Big( \sum_{k=1}^{k_0} \frac{(\alpha(2\nu + \epsilon_{k-1}, \epsilon_k/512) + \epsilon_k) (\nu + \epsilon_k)}{\epsilon_k^2} d (3\ln^2\frac{1}{\epsilon} + 2\ln^2\frac{1}{\delta}) + \frac{ \P_U(x \in \Delta(2\nu+\epsilon_{k-1}))}{\epsilon_k} (2d'\ln\frac{1}{\epsilon} + \ln\frac{1}{\delta}) \Big) \\
&\leq& (\sup_{r \geq \epsilon} \frac{\alpha(2\nu+r,r/1024)+r}{2\nu+r}) d (3\ln^2\frac{1}{\epsilon} + 2\ln^2\frac{1}{\delta}) \sum_{k=0}^{k_0} \frac{(\nu + \epsilon_k)^2}{\epsilon_k^2} + \sup_{r \geq \epsilon} \frac{ \P_U(x \in \Delta(2\nu+r))}{2\nu+r} (2d'\ln\frac{1}{\epsilon} + \ln\frac{1}{\delta}) \sum_{k=0}^{k_0} \frac{(\nu + \epsilon_k)}{\epsilon_k} \\
&\leq& \tilde{O}((\sup_{r \geq \epsilon} \frac{\alpha(2\nu+r,r/1024)+r}{2\nu+r}) d(\frac{\nu^2}{\epsilon^2} + 1) + (\sup_{r \geq \epsilon} \frac{\P_U(x \in \Delta(2\nu+r))}{2\nu+r}) d'(\frac{\nu}{\epsilon} + 1))
\end{eqnarray*}
\normalsize
where the first inequality is by Lemma~\ref{lem:inductive}, the second inequality is by noticing for all $k \geq 1$, $\ln^2\frac{1}{\epsilon_k} + \ln^2\frac{1}{\delta_k} \leq 3\ln^2\frac{1}{\epsilon} + 2\ln^2\frac{1}{\delta}$ and $d'\ln\frac{1}{\epsilon_k} + \ln\frac{1}{\delta_k} \leq 2d'\ln\frac{1}{\epsilon} + \ln\frac{1}{\delta}$, the rest of the derivations follows from standard algebra.

\end{proof}

\begin{proof}[Proof of Theorem~\ref{thm:labelcomplexity}]
Item 1 is an immediate consequence of Lemma~\ref{lem:inductive}, whereas item 2 is a consequence of Theorem~\ref{thm:labelcomplexityconcrete}.
\end{proof}

\section{Case Study: Linear Classfication under Uniform Distribution over Unit Ball}

We remind the reader the setting of our example in Section~\ref{sec:theory}. $\calH$ is the class of homogeneous linear separators on the $d$-dimensional unit ball and $\calH^{df}$ is defined to be $\{ h \Delta h' : h, h' \in \calH \}$. Note that $d'$ is at most $5d$. Furthermore, $U$ is the uniform distribution over the unit ball. $O$ is a deterministic labeler such that $\err_D(h^*) = \nu > 0$, $W$ is such that there exists a difference classifier $\bar{h}^{df}$ with false negative error $0$ for which $\Pr_U(\bar{h}^{df}(x) = 1) \leq g = o(\sqrt{d}\nu)$. We prove the label complexity bound provided by Corollary~\ref{cor:example}.

\begin{proof}[Proof of Corollary~\ref{cor:example}]
We claim that under the assumptions of Corollary~\ref{cor:example}, $\alpha(2\nu+r,r/1024)$ is at most $g$. Indeed, by taking $h^{df} = \bar{h}^{df}$, observe that
\[ P(\bar{h}^{df}(x) = -1, y_W \neq y_O, x \in \Delta(2\nu+r)) \leq P(\bar{h}^{df}(x) = -1, y_W \neq y_O) = 0 \]
\[ P(\bar{h}^{df}(x) = +1,  x \in \Delta(2\nu+r)) \leq g \]
This shows that $\alpha(2\nu+r, 0) \leq g$. Hence, $\alpha(2\nu+r, r/1024) \leq \alpha(2\nu+r, 0) \leq g$.
Therefore, 
\[ \sup_{r: r \geq \epsilon} \frac{\alpha(2\nu+r,r/1024)+r}{2\nu+r} \leq \sup_{r \geq \epsilon}\frac{g+r}{\nu+r} \leq \max(\frac{g}{\nu}, 1) \]
Recall that the disagreement coefficient $\theta(2\nu+r) \leq \sqrt{d}$ for all $r$, and $d'\leq 5d$. Thus, by Theorem~\ref{thm:labelcomplexity}, the number of label queries to $O$ is at most
\[ \tilde{O}\left( d \max(\frac{g}{\nu}, 1) (\frac{\nu^2}{\epsilon^2} + 1) + d^{3/2} \left(1 + \frac{\nu}{\epsilon}\right) \right) \] 
\end{proof}

\section{Performance Guarantees for Learning with Respect to Data labeled by $O$ and $W$}

An interesting variant of our model is to consider learning from data labeled by a mixture of $O$ and $W$. 

Let $D_W$ be the distribution over labeled examples determined by $U$ and $W$, specifically, $\P_{D_W}(x,y) = \P_U(x)\P_W(y|x)$. Let $D'$ be a mixture of $D$ and $D_W$, specifically $D' = (1-\beta)D + \beta D_W$, for some parameter $\beta > 0$. Define $h'$ to be the best classifier with respect to $D'$, and denote by $\nu'$ the error of $h'$ with respect to $D'$.

Let $O'$ be the following {\em{mixture oracle}}. Given an example $x$, the label $y_{O'}$ is generated as follows. $O'$ flips a coin with bias $\beta$. If it comes up heads, it queries $W$ for the label of $x$ and returns the result; otherwise $O$ is queried and the result returned. It is immediate that the conditional probability induced by $O'$ is $P_{O'}(y|x) = (1-\beta)\P_O(y|x) + \beta \P_W(y|x)$, and $D'(x,y) = P_{O'}(y|x) P_U(x)$.

\begin{assumption} \label{ass:boundedfnmixture}
For any $r, \eta > 0$, there exists an $h^{df}_{\eta, r} \in \calH^{df}$ with the following properties:
\begin{eqnarray*}
& \P_{\calD}(h^{df}_{\eta, r}(x) = -1 , x \in \Delta(r) , y_{O'} \neq y_W)  \leq  \eta \label{eqn:fnmixture} \\
& \P_{\calD}(h^{df}_{\eta, r}(x) = 1 ,  x \in \Delta(r))  \leq  \alpha'(r, \eta) \label{eqn:posmixture}
\end{eqnarray*}
\end{assumption} 

Recall that the disagreement coefficient $\theta(r)$ at scale $r$ is $\theta(r) = \sup_{h \in \calH} \sup_{r' \geq r} \frac{\P_U(\DIS(\B_U(h, r'))}{r'}$, which only depends on the unlabeled data distribution $U$ and does not depend on $W$ or $O$.

We have the following corollary. 

\begin{corollary} [Learning with respect to Mixture] \label{cor:beta}
Let $d$ be the VC dimension of $\calH$ and let $d'$ be the VC dimension of $\calH^{df}$. If Assumption~\ref{ass:boundedfnmixture} holds, and if the error of the best classifier in $\calH$ on $D'$ is at most $\nu'$. Algorithm~\ref{alg:main} is run with inputs unlabeled distribution $U$, target excess error $\epsilon$, confidence $\delta$, labeling oracle $O'$, weak oracle $W$, hypothesis class $\calH$, hypothesis class for difference classifier $\calH^{df}$, confidence $\delta$. Then with probability $\geq 1 - 2\delta$, the following hold:
\begin{enumerate}
\item the classifier $\hat{h}$ output by Algorithm~\ref{alg:main} satisfies: $\err_{D'}(\hat{h}) \leq \err_{D'}(h') + \epsilon$. 
\item the total number of label queries made by Algorithm~\ref{alg:main} to the oracle $O$ is at most:
\begin{eqnarray*}
\tilde{O} \Big{(}  (1-\beta) \Big(\sup_{r \geq \epsilon} \frac{\alpha'(2 \nu' + r, r/1024)+r}{2 \nu' + r} \cdot d \left( \frac{\nu'^2}{\epsilon^2} + 1 \right) + \theta(2 \nu' + \epsilon) d' \left(\frac{\nu'}{\epsilon} + 1\right) \Big) \Big{)}
\end{eqnarray*}
\end{enumerate}
\end{corollary}

\begin{proof}[Proof Sketch]
Consider running Algorithm~\ref{alg:main} in the setting above. By Theorem~\ref{thm:consistency} and Theorem~\ref{thm:labelcomplexity}, there is an event $F$ such that $\P(F) \geq 1-\delta$, if event $F$ happens, $\hat{h}$, the classifier learned by Algorithm~\ref{alg:main} is such that
\[ \err_{D'}(\hat{h}) \leq \err_{D'}(h') + \epsilon \]
By Theorem~\ref{thm:labelcomplexity}, the number of label requests to $O'$ is at most
\begin{eqnarray*}
m_{O'} = \tilde{O} \Big{(} \sup_{r \geq \epsilon} \frac{\alpha'(2 \nu' + r, r/1024) + r}{2 \nu' + r} \cdot d \left( \frac{\nu'^2}{\epsilon^2} + 1 \right) + \theta(2 \nu' + \epsilon) d' \left(\frac{\nu'}{\epsilon} + 1\right)  \Big{)}
\end{eqnarray*}
Since $O'$ is simulated by drawing a Bernoulli random variable $Z_i \sim \text{Ber}(1-\beta)$ in each call of $O'$, if $Z_i = 1$, then return $O(x)$, otherwise return $W(x)$.
Define event 
\[ H = \{ \sum_{i=1}^{m_{O'}} Z_i \leq 2 ((1-\beta) m_{O'} + 4\ln\frac{2}{\delta}) \}\] 
by Chernoff bound, $\P(H) \geq 1-\delta$. 
Consider event $J = F \cap H$, by union bound, $\P(J) \geq 1-2\delta$. Conditioned on event $J$, the number of label requests to $O$ is at most $\sum_{i=1}^{m_{O'}} Z_i$, which is at most
\begin{eqnarray*}
\tilde{O} \Big{(} (1 - \beta) \Big(\sup_{r \geq \epsilon} \frac{\alpha'(2 \nu' + r, r/1024)+r}{2 \nu' + r} \cdot d \left( \frac{\nu'^2}{\epsilon^2} + 1 \right) + \theta(2 \nu' + \epsilon) d' \left(\frac{\nu'}{\epsilon} + 1\right) \Big) \Big{)}
\end{eqnarray*}

\end{proof}

\section{Remaining Proofs}
\label{sec:remain}

\begin{proof}[Proof of Fact~\ref{fact:ek1}]

(1) 
First by Lemma~\ref{lem:adaptivebias}, $\P_{\calD}(x \in R_{k-1})/2 \leq \hat{p}_k \leq \P_{\calD}(x \in R_{k-1})$ holds with probability $1 - \delta_k/6$.\\
Second, for each classifier $h^{df} \in \calH^{df}$, define functions $f_{h^{df}}^1$, and $f_{h^{df}}^2$ associated with it. Formally,
\[ f_{h^{df}}^1(x,y_O,y_W) = I(h^{df}(x) = -1 , y_O \neq y_W) \]
\[ f_{h^{df}}^2(x,y_O,y_W) = I(h^{df}(x) = +1) \]
Consider the function class $\calF^1 = \{ f_{h^{df}}^1: h^{df} \in \calH^{df} \}$, $\calF^2 = \{ f_{h^{df}}^2: h^{df} \in \calH^{df} \}$. Note that both $\calF^1$ and $\calF^2$ have VC dimension $d'$, which is the same as $\calH^{df}$. We note that $\calA_k'$ is a random sample of size $m_k$ drawn iid from $\calA_k$. The fact follows from normalized VC inequality on $\calF^1$ and $\calF^2$ and the choice of $m_k$ in Algorithm~\ref{alg:traindc} called in epoch $k$, along with union bound.
\end{proof}

\begin{proof}[Proof of Fact~\ref{fact:ek2}]
For fixed $t$, we note that $S_k^t$ is a random sample of size $2^t$ drawn iid from $D$. By Equation~\eqref{eqn:normalizedvc}, for any fixed $t \in \N$,
\small
\begin{equation} 
\P \Big( \text{ for all } h, h' \in \calH, |(\err(h, S_k^t) - \err(h', S_k^t)) - (\err_D(h) - \err_D(h'))| \leq \sigma(2^t,\delta_k^t) + \sqrt{\sigma(2^t,\delta_k^t) \rho_{S_k^t}(h,h')} \Big) 
\geq 1 - \delta_k^t/8 
\label{eqn:ek31}
\end{equation}
\normalsize
Meanwhile, for fixed $t \in \N$, note that $\hat{S}_k^t$ is a random sample of size $2^t$ drawn iid from $\hat{D}_k$. By Equation~\eqref{eqn:normalizedvc},
\begin{equation} 
\P\Big( \text{ for all } h, h' \in \calH, \err(h, \hat{S}_k^t) - \err_{\hat{D}_k}(h) \leq \sigma(2^t,\delta_k^t) + \sqrt{\sigma(2^t,\delta_k^t)\err_{\hat{D}_k}(h)} \Big) \geq 1 - \delta_k^t/8
\label{eqn:ek32}
\end{equation}
Moreover, for fixed $t \in \N$, note that $\calS_k^t$ is a random sample of size $2^t$ drawn iid from $\calD$. By Equation~\eqref{eqn:normalizedchernoff},
\begin{eqnarray}
&&\P \Big( \P_{\calS_k^t}(\hat{h}^{df}_k(x) = -1 , y_O \neq y_W , x \in R_{k-1}) \leq \P_{\calD}(\hat{h}^{df}_k(x)= -1 , y_O \neq y_W , x \in R_{k-1}) \nonumber\\
&+& \sqrt{\gamma(2^t,\delta_k^t)\P_{\calD}(\hat{h}^{df}_k(x) = -1 , y_O \neq y_W , x \in R_{k-1}) } + \gamma(2^t,\delta_k^t) \Big) \geq 1 - \delta_k^t/8
\label{eqn:ek33}
\end{eqnarray}
Finally, for fixed $t \in N$, note that $\calS_k^t$ is a random sample of size $2^t$ drawn iid from $\calD$. By Equation~\eqref{eqn:normalizedchernoff},
\small
\begin{equation} 
\P \Big( \P_{\calS_k^t}(\hat{h}^{df}_k(x)=-1 , x \in R_{k-1}) \leq \P_{\calD}(\hat{h}^{df}_k(x)=-1 , x \in R_{k-1}) + \sqrt{\P_{\calD}(\hat{h}^{df}_k(x)=-1 , x \in R_{k-1})\gamma(2^t,\delta_k^t)}+ \gamma(2^t,\delta_k^t) \Big) \geq  1-\delta_k^t/8
\end{equation}
\normalsize
Note that by algebra,
\small
\[ \P_{\calD}(\hat{h}^{df}_k(x)=-1 , x \in R_{k-1}) + \sqrt{\P_{\calD}(\hat{h}^{df}_k(x)=-1 , x \in R_{k-1})\gamma(2^t,\delta_k^t)}+ \gamma(2^t,\delta_k^t) \leq 2(\P_{\calD}(\hat{h}^{df}_k(x)=-1 , x \in R_{k-1}) + \gamma(2^t,\delta_k^t)) \]
\normalsize
Therefore,
\begin{eqnarray} 
\P \Big( \P_{\calS_k^t}(\hat{h}^{df}_k(x)=-1 , x \in R_{k-1}) \leq 2( \P_{\calD}(\hat{h}^{df}_k(x)=-1 , x \in R_{k-1}) + \gamma(2^t,\delta_k^t)) \Big) \geq 1-\delta_k^t/12 
\label{eqn:ek34}
\end{eqnarray}
The proof follows by applying union bound over Equations~\eqref{eqn:ek31},~\eqref{eqn:ek32},~\eqref{eqn:ek33} and~\eqref{eqn:ek34} and $t \in \N$.
\end{proof}
We emphasize that $\calS_k^t$ is chosen iid at random after $\hat{h}^{df}_k$ is determined, thus uniform convergence argument over $\calH^{df}$ is not necessary for Equations~\eqref{eqn:ek33} and~\eqref{eqn:ek34}. 

\begin{proof}[Proof of Fact~\ref{fact:fk}]
By induction on $k$. 
\paragraph{Base Case.} For $k = 0$, it follows directly from normalized VC inequality that $\P(F_0) \geq 1 - \delta_0$.
\paragraph{Inductive Case.} Assume $\P(F_{k-1}) \geq 1 - \delta_0 - \ldots - \delta_{k-1}$ holds. By union bound,
\[ \P(F_k) \geq \P(F_{k-1} \cap E_k^1 \cap E_k^2) \geq \P(F_{k-1}) - \delta_k/2 - \delta_k/2 \geq \P(F_{k-1}) - \delta_k \] 
Hence, $\P(F_k) \geq 1 - \delta_0 - \ldots - \delta_k$. This finishes the induction.
\\
In particular, $\P(F_{k_0}) \geq 1 - \delta_0 - \ldots \delta_{k_0} \geq 1-\delta$.
\end{proof}

\end{document}